\documentclass[12pt,a4article]{article}

\usepackage{amsthm}

\usepackage{amsmath}
\usepackage{amsfonts}
\usepackage{amssymb}
\usepackage{graphicx}
\usepackage{caption}
\usepackage{subcaption}
\usepackage{MnSymbol,wasysym}
\usepackage{cancel}
\usepackage{float}
\usepackage{color}
\usepackage{yfonts}
\usepackage{enumitem}
\theoremstyle{definition}
\newtheorem{defn}{Definition}[section]
\newtheorem{thm}{Theorem}[section]
\newtheorem{lem}[thm]{Lemma}

\theoremstyle{remark}
\newtheorem{remark}{Remark}[section]
\usepackage{multicol}
\usepackage{url}
\usepackage[title,toc,titletoc,page]{appendix}

\author{Ilana Segall and Alfred M. Bruckstein\thanks{This research was in part supported by Technion Autonomous Systems Program (TASP)}}  \date{Center for Intelligent Systems\\MultiAgent Robotic Systems (MARS) Lab\\Computer Science Department\\
Technion, Haifa 32000, Israel\\
\today}

\title{Broadcast Guidance of Agents in Deviated Linear Cyclic Pursuit  }

 \newtheorem{theorem}{\textit{Theorem}}
 \newtheorem{lemma}{\textit{Lemma}}
\newtheorem{corollary}{\textit{Corollary}}

\newtheorem{definition}{\textit{Definition}}

\begin{document}
\maketitle
\newpage

\tableofcontents

\newpage

\begin{abstract}
In this report we show the emergent behavior of a group of agents, ordered from 1 to $n$, performing deviated, linear, cyclic pursuit, in the presence of a broadcast guidance control.
Each agent senses the relative position of its target, i.e. agent $i$ senses the relative position of agent $i+1$. 
 The broadcast control, a velocity signal, is detected by a random set of agents in the group. We assume the agents to be  modeled  as single integrators.   We show that the emergent behavior of the group is determined by the deviation angle and by the set of agents detecting the guidance control. 

\end{abstract}

\section{Introduction}

The work presented in this report is a first follow-up of the report "Guidance of Agents in Cyclic Pursuit",  see  \cite{S-B-arxiv-2020}, where the problem of (direct)  linear and non-linear cyclic pursuit, in the presence of a broadcast velocity control detected by a random set of  agents in the group,  has been thoroughly investigated. In the existing literature, "cyclic pursuit" is meant to be an \textbf{autonomous} system of $n$ agents, ordered
from 1 to $n$, which behaves according to the following rule: agent $i$ chases agent $i+1$ and agent $n$ chases agent $1$. In the sequel \textbf{all indices associated with agents are $modulo(n)$} and agent $i+1$ is defined as the target of agent $i$.

 As in our previous work, we assume the agents to be identical, memory-less, particles, modeled as single integrators. Each  agent  can sense the relative position of its target but its own absolute location is unknown.  The orientation of all local coordinate systems is aligned to that of the global coordinate system,  i.e. agents are assumed to have compasses enabling them to align their local reference frames to a global reference direction (a common north).

In autonomous linear deviated cyclic pursuit,  each agent  moves in a direction rotated by an angle $\theta$ from the line of sight to its target. 

Let $p_i(t)$ be the position of agent $i$ at time $t$;  $p_i(t) \in \mathbb{R}^2$. 
Then, the rule of movement for agent $i$, in the autonomous system is 
\begin{equation}\label{dot_p_i-h-R}
 \dot{p}_i(t) =R(\theta) (p_{i+1}(t) - p_i(t))
\end{equation}
where  $R(\theta)$ is the rotation matrix 
\begin{equation}\label{eq-R}
 R(\theta) = \left [
\begin{matrix}
cos(\theta) & sin(\theta)\\
-sin(\theta) & cos(\theta)
\end{matrix} \right ]
\end{equation}
This rule of movement is illustrated in Fig. \ref{fig-theta_def} where $u_i(t)= (p_{i+1}(t) - p_i(t))$. 
\begin{figure}[H]
\begin{center}
\includegraphics[scale=0.5]{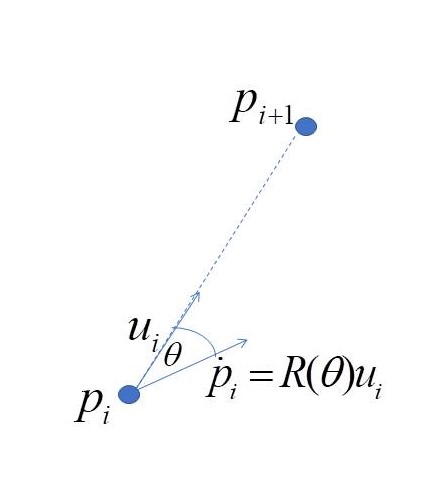}
\caption{Illustration of deviated pursuit }\label{fig-theta_def}
\end{center}
\end{figure} 
The deviation angle $\theta$ is assumed to be constant and common to all agents.

\subsection{Literature survey}
Various researchers have treated the problem of formations obtained in  \emph{autonomous cyclic pursuit},  without external input. In the sequel we refer to those considering agents modeled as single-integrators with linear dynamics.  
The problem of  linear agreement algorithms for agents with single integrator kinematics, over general topologies, was vastly investigated (see e.g. \cite{JLM},\cite{OS-M2007}, \cite{OS-M2003}, \cite{S-B-arxiv}). Cyclic pursuit is a particular unidirectional topology where each agent receives information from a single neighbor, its target.

Lin et al. consider in \cite{LBF2003} $n$ ordered and numbered points in the complex plane, each representing a freely mobile agent. They show that when the agents are in linear cyclic pursuit they converge to    
the centroid of the points and the centroid is stationary. Moreover, 
 other formations are achievable by a simple modification, where
each agent chases a "fixed" displacement of its target.

Sinha and Goose, \cite{SinGoo07}, consider a group of autonomous mobile agents, with  identical behavior rules but with generally different control gains
\begin{eqnarray*}
\dot{p}_i &=& u_i\\
u_i &=& k_i (p_{i+1}-p_i)
\end{eqnarray*}
for $i= 1,\dots, n$.
They show that, by suitably selecting the gains $k_i$, the collective behavior of the agents can be controlled
 to obtain not only  a point of convergence but also directed motion, where  the trajectories of all the agents converge to a
straight line as $t \rightarrow	\infty$, (see conditions in Theorem 3,  Theorem 4). 

Pavone and Frazzoli, \cite{PavFraz}, generalize the (linear) cyclic pursuit strategy to deviated cyclic pursuit, in the context of geometric pattern formation.
They show that $n$ agents  operating in $\mathbb{R}^2$, in deviated cyclic pursuit,  with a common offset angle  $\theta$, will 
eventually converge to a single point, a circle or an ever expanding logarithmic spiral pattern, depending on the value of  $\theta$, as follows:
\begin{itemize}
\item if $|\theta| < \pi/n$, to a single point, determined by the initial positions of the agents (the initial centroid)
\item if $|\theta| = \pi/n$, to an evenly spaced circle formation, with radius
 determined by the initial positions of the agents
\item if $ \pi/n < |\theta| < 2\pi/n$, to an evenly spaced, ever expanding, logarithmic spiral formation.
\end{itemize}

Ramirez et al. in \cite{RamirezPhD}, \cite{RPFM2009}, extended these results in three directions:
\begin{enumerate}
\item the agents move in $\mathbb{R}^3$
\item control of the center of the formation by modifying the autonomous velocity control
\begin{eqnarray*}
u_i &=& R(\theta) (p_{i+1} -p_i) - k_c p_i; \quad k_c \geq 0\\
\dot{p}_i &=& u_i
\end{eqnarray*}
If $k_c >0$ then 
\begin{enumerate}
\item the center of the formation is no longer determined
by the initial positions of the agents, instead it always converges, exponentially fast, to the origin. 
\item each agent is required to know not only the relative position to its target but also its own position, (in contradiction to our requirements here).
\end{enumerate}

\item control the radius of the  evenly-spaced circular formations.\\
The key idea here
is to relax the assumption of a common deviaton angle, $\theta$, and make the deviation angle of each agent, $\theta_i$,  a function of the location of the agents and the desired inter-agent distance. In this case, the rule of motion (\ref{dot_p_i-h-R}) becomes
\begin{equation*}
\dot{p}_i=R(\theta_i)(p_{i+1}-p_i)
\end{equation*}
where $$\theta_i =\frac{\pi}{n}+k_{th} (r- \|p_{i+1}-p_i\|); \quad r, k_{th} > 0$$
and $r$ is the desired inter-agent distance and $k_{th}$ is a gain. 
\end{enumerate}

Ren considered more general topologies and shows that collective motions including rendezvous, circular
patterns, and logarithmic spiral patterns can be achieved by
introducing Cartesian coordinate coupling to existing consensus
algorithms.  A particular case of coupling is the rotation matrix.
In  \cite{Ren2008} the effect of coordinates coupling, introduced by a rotation matrix, on linear consensus algorithms for agents with single integrator kinematics was considered.
In this case, the rule of movement is 
\begin{equation}\label{p_i-Ren}
\dot{p}_i = \sum_{k \in N_i} R(\theta) (p_k-p_i)
\end{equation}
where $R(\theta)$ is the rotation matrix and $N_i$ is the neighborhood of agent $i$, representing a \emph{general network topology}, in contrast to \cite{PavFraz}, where the analysis of a similar problem was limited to unidirectional
ring topology (cyclic pursuit).   
It is shown that both the network topology and the value
of the rotation angle affect the resulting collective motions.
When the nonsymmetric Laplacian matrix, representing the network topology, has
certain properties and the rotation angle is below, equal, or
above a critical value, the agents will eventually rendezvous,
move on circular orbits, or follow logarithmic spiral curves.
In particular, when the agents eventually move on circular
orbits, the relative radius of the orbits  is equal to the relative
magnitude of the components
of a right eigenvector associated with a critical eigenvalue of
the nonsymmetric Laplacian matrix, in contrast to the common circular obits obtained in case of cyclic pursuit (see    \cite{PavFraz}).

\subsection{An Overview of The New Results}

The contribution of this work is in the analysis of the impact of an exogenous velocity control on the emergent behavior of $n$ agents performing (autonomous) \emph{deviated} linear cyclic pursuit. The exogenous velocity control, we assume, is broadcast by a controller and detected by a random set of agents from the group, that thereby become temporary (ad-hoc) "leaders" of the swarm.

In this case,  the rule of movement (\ref{dot_p_i-h-R}) becomes 
\begin{equation}\label{GenWithControl-1}
  \dot{p}_i(t) = R(\theta)  (p_{i+1}(t) - p_i(t)) +  b_i(t) U_c(t); \quad i=1,...,n
\end{equation}
where
\begin{itemize}
   \item $U_c(t)$ is the broadcast velocity vector signal at time $t$
  \item $b_i(t)$ is the indicator of detection of $U_c(t)$ by agent $i$, given by  
\begin{equation}\label{eq-b_i}
b_i = \begin{cases}
1 \quad & \text{if agent  } $i$ \text{  detects  } U_c\\
0 \quad & \text{otherwise}
\end{cases}
\end{equation}
  
\end{itemize}

\subsubsection{Group dynamics with broadcast control}\label{GroupDyn}

Let $P=\left [p_1, p_2,...p_n \right ]^T$ be the vector of stacked agent positions, where $p_i, \quad i=1, \dots, n$ evolves according to eq. (\ref{GenWithControl-1}). The group dynamics of the system can  be written as:
\begin{equation}\label{P-LCR}
  \dot{P}(t)= \hat{M} P(t)  +\hat{B}(t) U_c(t)
\end{equation}
where 
\begin{itemize}
   \item $\hat{M}=M \otimes R(\theta)$, such that 
    $M$ is the circulant matrix (\ref{eq-M}):
 \begin{align}\label{eq-M}
   M_{n\times n}= circ[ -1, 1, 0, 0, \hdots, 0]=
          \begin{bmatrix}
           -1 & 1 & 0 & 0 & \hdots & 0 \\
            0 & -1 & 1 & 0 & \hdots & 0 \\
	   &&\vdots \\
            1 & 0  && \hdots& 0  & -1 \\
          \end{bmatrix}
  \end{align}
  and $R(\theta)$ is defined by (\ref{eq-R}).
  \item $\hat{B}(t) = B(t) \otimes I$, where $B$ is the "leaders" indicator, i.e. $B_{n \times 1}(i)=b_i$, where $b_i$ is defined by (\ref{eq-b_i})  and $I$ is the $2 \times 2$ identity matrix 
   \item $\otimes$ denotes the Kronecker product, see Appendix \ref{Kron-def} 
   \item $U_c$ is the external broadcast control, $U_c(t) \in \mathbb{R}^2$
 \end{itemize} 
   Due to the structure of $M$,  $\hat{M}$ is block circulant, as follows:
  \begin{align}\label{M_hat}
\hat{M} =
             \begin{bmatrix}
          - R(\theta) & R(\theta) & 0_{2 \times 2}&0_{2 \times 2}&\hdots & 0_{2 \times 2} \\ \\
           0_{2 \times 2} &  - R(\theta) & R(\theta) & 0_{2 \times 2}&\hdots &0_{2 \times 2} \\
          \\
	   &&\vdots \\
	\\
         R(\theta)   & 0_{2 \times 2} & \hdots &&0_{2 \times 2}& -R(\theta)
          \end{bmatrix}
  \end{align}

The matrix $M$  is time independent by its definition (see (\ref{eq-M})). The assumption of cyclic pursuit with a constant deviation angle makes $R(\theta)$,  and therefore $\hat{M}$ too,  time independent. We assume the exogenous control, $U_c(t)$, and the set of agents detecting it, defined by the leaders indicator vector $B(t)$, to be piecewise constant.   
 In the sequel we treat separately each
 interval where $U_c(t)$  and $B(t)$ ( the set of agents detecting the exogenous control) are constant and thus it is convenient  to let $t$ denote the relative time since the beginning of the interval and $P(0)$ 
the state of the system at the moment of change, defined to be $t=0$. 
\begin{remark}
The  behavior of the group of agents in multiple intervals, where $U_c$ and/or $B$ change at the start of each new interval and the end conditions of one interval are the start conditions of the next interval, is discussed and illustrated by simulation in section \ref{Sim-multi}.
\end{remark}   
In each interval 
 eq. (\ref{P-LCR}) becomes 
\begin{equation}\label{eq-LTI}
 \dot{P}(t)= \hat{M} P(t)  +\hat{B} U_c
\end{equation}
i.e. the system evolves as a linear time independent system (LTI) with the well known solution (\ref{eq-Pt}), (see \cite{TK}).
\begin{equation}\label{eq-Pt}
P(t) = e^{\hat{M} t} P(0) + \int_{0}^{t} e^{\hat{M}(t-\tau)} \hat{B} U_c \mathrm{d}\tau
\end{equation}

\subsubsection{The resulting emergent behavior}
Given an external broadcast control, $U_c$, the emergent behavior of the group is a function of the deviation angle $\theta$, the critical angle, $\displaystyle \theta_c=\frac{\pi}{n}$ and the (random) subset of agents detecting the broadcast control (leaders), represented by the leaders indicator vector $B$. We have the following results:
\begin{itemize}
\item if $\theta < \theta_c$ then
\begin{itemize}
\item if $B=\mathbf{1}_n$, i.e.  all agents detect the broadcast control signal, then the asymptotic behavior of the group is given by
$$\boxed{p_i(t \rightarrow \infty) = p_c+U_c t; \quad i=1, \dots, n}$$
where $p_c$ is the centroid of the initial positions. Thus, if  all agents detect the broadcast control, the agents will eventually  gather and move as a single point with velocity $U_c$ on a line anchored at $p_c$.
\item if $B \neq \mathbf{1}_n$ then
$$\boxed{p_i(t \rightarrow \infty) = p_c+\frac{n_l}{n} U_c t + \sigma_i R(-\theta) U_c; \quad i=1, \dots, n}$$
 In this case, all agents will eventually move with velocity  $\frac{n_l}{n} U_c $ on parallel lines anchored at $p_c+ \sigma_i R(-\theta) U_c$, where 
\begin{itemize}
\item $n_l$ is the number of agents detecting $U_c$.
\item $\sigma_i R(-\theta) U_c$ is the asymptotic deviation of the anchor of the movement line of agent $i$ from the centroid $p_c$ and $\sigma_i$ is a function of $B$, i.e. of the set of leaders (see section \ref{Derive-Dev})
\end{itemize} 
   Therefore,  the agents will asymptotically align in a linear formation, rotated by $-\theta$ from  the direction of $U_c$, and the formation will move with velocity $\displaystyle \frac{n_l}{n} U_c$. 
\end{itemize} 
\item if $\theta = \theta_c$ then
\begin{itemize}
\item if $B=\mathbf{1}_n$ then
$$\boxed{p_i(t \rightarrow \infty) = p_c+U_c t + p^r_i(t); \quad i=1, \dots, n}$$
where 
\begin{itemize}
\item $p^r_i(t)$ is a rotation component of the position at time $t$, s.t.
$$\boxed{p^r_i(t) = [r \sin (\omega t+\alpha_i), r \cos(\omega t+\alpha_i)]^T}$$ 
\item $\omega=2 \sin(\frac{\pi}{n})$
\item $R$ is the radius of a circular orbit, common to all agents, determined by the number of agents and the initial positions
\item The agents are equally spaced on the orbit, at angular distance $\displaystyle \frac{2 \pi}{n}$, i.e. $\displaystyle \alpha_{i+1}-\alpha_i = \frac{2 \pi}{n} $
\end{itemize}

Therefore, the agents are asymptotically moving, at equal distances, on a circular   orbit, around a common center moving with velocity $U_c$. 
\item if $B \neq \mathbf{1}_n$ then
$$\boxed{p_i(t \rightarrow \infty) = p_c+\frac{n_l}{n} U_c t + p^r_i(t) + \tilde{\sigma_i} R(-\theta) U_c; \quad i=1, \dots, n}$$
where $p^r_i(t)$ is the position of agent $i$ on a common circular orbit, same as above, but in this case the center of the orbit, for agent $i$  is shifted by $\tilde{\sigma_i} R(-\theta) U_c $ from a common center, where $\tilde{\sigma_i} $ is a function of $B$ (see section  \ref{Derive-Dev}). All centers move with velocity $\displaystyle \frac{n_l}{n} U_c$. 
\end{itemize}

\end{itemize}

\begin{remark}
We observe that the circular formation and its characteristics (radius and angular distances between agents) are independent of the broadcast control $U_c$ and the set of agents detecting it. These influence only the center of the orbit and its movement.
\end{remark}
\begin{remark}
The emergent behavior for the various deviation angles, summarized above, is determined by the eigenvalues of the matrix $\hat{M}$ (see Appendix \ref{eig-hat_M-1}). We note that if $\theta > \theta_c$ then then $\hat{M}$ has   at least two non-zero eigenvalues which lie in the open right-half complex plane causing the position of the agents to spiral out. This is an unstable system which is not discussed herein.
\end{remark} 
The emergent behaviors are rigorously derived in section \ref{P_t_inf} and illustrated by simulations in section \ref{Sim}.

\subsection{Preliminaries - Properties of the matrix $\hat{M}$ }\label{prop-hat_M}
The following properties of the matrix $\hat{M}$ are used in the sequel for the derivation of the emergent behavior of the  system, $P(t \rightarrow \infty)$,  evolving under (\ref{eq-LTI}).
\begin{itemize}
 \item $\hat{M}$ is  a normal matrix (see Lemma \ref{normal-M}), therefore it is unitarily diagonizable, i.e.
 \begin{equation*}
 \hat{M} =  \Gamma \Lambda \Gamma^*
\end{equation*}
where 
\begin{itemize}
\item $\Lambda$ is the diagonal matrix of eigenvalues of $\hat{M}$, denoted   by\\ $\mu_k^\pm; \quad k=0, ...., n-1$, (see Appendix \ref{eig-hat_M-1})
\begin{equation}\label{eq-mu_k}
\mu_k^{\pm} =2\sin(\frac{ \pi k}{n})e^{-j (\frac{\pi}{2}+\frac{ \pi  k}{n}\pm \theta)}
\end{equation} 
\item $\Gamma$ is the unitary matrix of eigenvectors of $\hat{M}$, denoted  by\\ $\zeta_k^\pm; \quad k=0, ...., n-1$, see Appendix \ref{eig-hat_M-1}, such that
\begin{equation}\label{zeta_k-pm}
\zeta_k^\pm = v_k \otimes (1,\pm j)^T
\end{equation}
where $$  v_k=\frac{1}{\sqrt{n}} \left ( 1, e^{-2 \pi jk/n}, e^{-4 \pi jk/n},\hdots, e^{-2 \pi jk(n-1)/n}  \right)^T; \quad k \in \left \{0, 1, \hdots, n-1 \right \}$$
\end{itemize}
\item  For $n$ agents in linear cyclic pursuit with common deviation angle $\theta$, there exists a critical value  $\theta_c=\frac{\pi}{n}$, (see Corollary \ref{eig-M-2} in Appendix \ref{comp-theta_c}), such that
\begin{itemize}
\item[(a)] if $|\theta| < \theta_c$, then $\hat{M}$ has two zero eigenvalues, $\mu_0^\pm$, and all other eigenvalues  lie in the open left-half complex plane
  \item[(b)] if  $|\theta| = \theta_c$, then $\hat{M}$ has two zero eigenvalues, $\mu_0^\pm$, and two non-zero eigenvalues which lie on the imaginary axis, while all remaining eigenvalues lie in the open left-half complex plane. The eigenvalues on the imaginary axis are 
  $\mu_{n-1}^+$ and  $\mu_1^-$       
  \item[(c)] if $|\theta| > \theta_c $, then $\hat{M}$ has  two zero eigenvalues and at least two non-zero eigenvalues which lie in the open right-half complex plane, therefore this is \textbf{an unstable system which shall not be discussed} herein. The agents in this case spread out and cease to remain in a finite extent constellation.
  
\end{itemize} 
  \end{itemize}

\section{Derivation of The  Emergent Behavior}\label{P_t_inf}
We consider a piecewise constant time interval of a system  evolving according to eq. (\ref{eq-LTI}), with a solution given by eq. (\ref{eq-Pt}).  Since the emergent behavior is given by $P(t \rightarrow \infty)$ we assume that the dwell time of the interval is long enough to approach the asymptotic behavior. 
We recall (see eq. (\ref{eq-Pt}))  
\begin{equation*}
P(t) = e^{\hat{M} t} P(0) + \int_{0}^{t} e^{\hat{M}(t-\tau)} \hat{B} U_c \mathrm{d}\tau
\end{equation*}
which can be decomposed into
\begin{equation}
P(t) = P^{(h)}(t) + P^{(u)}(t)\label{P-u}
\end{equation}
where 
\begin{itemize}
\item $\displaystyle P^{(h)}(t)= e^{\hat{M} t} P(0) $
 is the homogeneous part of the solution, representing the zero input dynamics of the swarm
 \item $\displaystyle P^{(u)}(t) = \int_{0}^{t} e^{\hat{M}(t-\tau)} \hat{B} U_c \mathrm{d}\tau$ is the contribution of the guidance control, $U_c$, to the dynamics of the swarm

\end{itemize}

\subsection{Zero input dynamics }\label{ZeroInp}

  \begin{eqnarray}
  P^{(h)}(t) &=& e^{\hat{M} t} P(0)\\ 
  &=& \frac{1}{2 }\sum_{i=0}^{n-1} [ \zeta^+_i e^{\mu^+_i t} (\zeta^+_i)^* +  \zeta^-_i e^{\mu^-_i t} (\zeta^-_i)^* ] P(0)\label{eq-Ph-t}
\end{eqnarray}
where we used the decomposition property of $\hat{M}$ and  $\mu_i^\pm, \zeta_i^\pm$ are given by equations (\ref{eq-mu_k}) and (\ref{zeta_k-pm}) respectively.

Separating in (\ref{eq-Ph-t}) the elements with zero eigenvalues  from the remaining elements, we can rewrite (\ref{eq-Ph-t}) as
\begin{equation}\label{P-h}
  P^{(h)}(t) = P^{(h_0)} + \Delta_P^{(h_0)}(t)
\end{equation}
where
\begin{itemize}
\item $P^{(h_0)}$ is the  component of the homogeneous position vector $P^{(h)}(t)$ due to the zero eigenvalues 
\begin{equation}\label{eq-Ph-1}
P^{(h_0)} = \frac{1}{2 }[ \zeta^+_0(\zeta^+_0)^* +  \zeta^-_0 (\zeta^-_0)^*  ] P(0)
\end{equation}
$P^{(h_0)}$ 
 is time independent and common to all $\theta$, see section \ref{derive-Ph0}
\item $ \Delta_P^{(h_0)}(t)$ is the  component of $P^{(h)}(t)$, due to the remaining eigenvalues 
\begin{equation}\label{Delta-Ph-1}
\Delta_P^{(h_0)}(t)  =  \frac{1}{2}\sum_{i=1}^{n-1} [ \zeta^+_i e^{\mu^+_i t} (\zeta^+_i)^* +  \zeta^-_i e^{\mu^-_i t} (\zeta^-_i)^* ] P(0)
\end{equation}
$\Delta_P^{(h_0)}(t)$
 is time dependent and its asymptotic value depends on $\theta$, see section \ref{derive-Delta-Ph0}  
\end{itemize}

\subsubsection{The derivation of $P^{(h_0)}$}\label{derive-Ph0}
\begin{lem}\label{p-h1}
Let $P^{(h_0)}:= (p_1^{(h_0)}, \dots, p_{n}^{(h_0)})^T $ where $P^{(h_0)}$ is defined by  eq. (\ref{eq-Ph-1}).
Denote by $p_c =(x_c, y_c)^T $ the centroid of the agents' initial positions, $p_i(0)=(x_i(0),y_i(0))^T$. 

Then
\begin{equation}\label{eq-Pc}
  p_i^{(h_0)} = p_c; \quad \forall i \in [1, \dots, n]
\end{equation}
\end{lem}
\begin{proof}
Recalling that the two eigenvectors  of $\hat{M}$,  corresponding to the two zero eigenvalues are $\zeta_0^\pm = v_0 \otimes [1, \pm j]^T$ where $v_0=\mathbf{1}_n$ we have
\begin{equation}\label{zeta_1}
   \zeta_0^\pm = \frac{1}{\sqrt{n}}[1, \pm j,1,\pm j, \dots, 1, \pm j]^T
\end{equation}
Let $Y^{(0)}=  \zeta^+_0(\zeta^+_0)^* $. Then, eq.  (\ref{eq-Ph-1}) can be rewritten as 
\begin{equation*}
 P^{(h_0)} = \frac{1}{2 }(Y^{(0)}+\overline{Y^{(0)}})P(0)=  \Re(Y^{(0)}) P(0)
\end{equation*}
\begin{align}
\Re(Y^{(0)}) = \frac{1}{n}\begin{bmatrix}
           1 & 0 &1 & 0 \dots &1 & 0 \\
           0 & 1 & 0 &1 \dots &0 &1 \\
           \vdots & & & \dots & & \vdots\\
            1 & 0 &1 & 0 \dots &1 & 0 \\
            0 & 1 & 0 &1 \dots &0 &1
          \end{bmatrix}\label{ReY}
  \end{align}
Thus,
\begin{equation}\label{eq-Ph-2}
 p_i^{(h_0)} = (\frac{1}{n}\sum_{i=0}^{n-1} x_i(0), \frac{1}{n}\sum_{i=0}^{n-1} y_i(0))^T \equiv p_c; \quad \forall i \in [1, \dots, n]
\end{equation}
\end{proof}
 \subsubsection{Derivation of the asymptotic behavior of $\Delta_{P}^{(h_0)}(t)$}\label{derive-Delta-Ph0} 
$\Delta_{P}^{(h_0)}(t)$ is the increment of the position, at time $t$, due to the eigenvalues other than the zero eigenvalues, see (\ref{P-h}).  
 \begin{lem}\label{L-Delta-p_h1}
 
Let
\begin{equation}\label{Delta_P_h0}
\Delta_P^{(h_0)}(t)= ( \Delta_{p_1}^{(h_0)}(t), \dots , \Delta_{p_{n}}^{(h_0)}(t)  )^T
\end{equation}   
 where $\Delta_{P}^{(h_0)}(t)$ is defined as in eq. (\ref{Delta-Ph-1}).
Then $\Delta_P^{(h_0)}(t \rightarrow \infty)$ is a function of $\theta$, as follows:

\begin{enumerate}[label=\textbf{L.\arabic*},ref=L.\arabic*]
\item if $|\theta| < \theta_c$ then  $\Delta_{p_k}^{(h_0)}(t)\xrightarrow{t \to \infty} \mathbf{0}_{2}; \forall k \in [1,\dots,n]$
\item  if $|\theta| = \theta_c$  then 
the agents asymptotically move, equally spaced, on a common circular orbit.
\begin{itemize}
\item The angular velocity is $\omega=2 \sin(\frac{\pi}{n})$
\item  The  angular distance between consecutive agents is $\frac{2 \pi}{n}$
\item The radius of the circular orbit is determined by the number of agents,$n$, and their initial positions $p_l(0);l=1,\dots,n$
\end{itemize} 

\end{enumerate}

\end{lem}
\begin{proof}\

\begin{enumerate}[label=\textbf{L.\arabic*},ref=L.\arabic*]
  \item  if \textbf{\underline{$|\theta|<\theta_c$ }},  all  non-zero eigenvalues of $\hat{M}$ have negative real part, see Corollary \ref{eig-M-2}.(a),  in Appendix \ref{comp-theta_c}
  
  Thus, from eq. (\ref{Delta-Ph-1}) we have  $\Delta_P^{(h_0)}(t) \xrightarrow{t \to \infty} \mathbf{0}_{2n}$. Therefore, according to  (\ref{Delta-Ph-1}), $$\Delta_{p_k}^{(h_0)}(t)\xrightarrow{t \to \infty} \mathbf{0}_{2}; \quad k=1, \dots, n$$
 
  \item  if \textbf{\underline{$|\theta|=\theta_c$ }}, then in addition to the 2 zero eigenvalues,  $\hat{M}$ has 2 eigenvalues on the imaginary axis $\mu_{n-1}^+, \mu_1^-$  and all remaining eigenvalues of $\hat{M}$ have negative real part. Thus, $\Delta_P^{(h_0)}(t)$ can be separated into a part containing the eigenvalues on the imaginary axis, denoted by $P^{(h_1)}(t)$, and a part containing the remaining eigenvalues, which have negative real part, denoted by $P^{(h_2)}(t)$. 
  \begin{equation*}
  \Delta_P^{(h_0)}(t) = P^{(h_1)}(t)+P^{(h_2)}(t)
  \end{equation*}
  Since $P^{(h_2)}(t)$ contains only eigenvalues with negative real part, we have $$P^{(h_2)}(t) \xrightarrow{t \to \infty} \mathbf{0}_{2n}$$ and 
\begin{equation}\label{P-circ}
 \Delta_P^{(h_0)}(t) \xrightarrow{t \to \infty} P^{(h_1)}(t) = \frac{1}{2} [e^{\mu^+_{n-1} t} \zeta^+_{n-1} (\zeta^+_{n-1})^* + e^{\mu^-_1 t} \zeta^-_1 (\zeta^-_1)^*] P(0)
\end{equation}

We show, in Appendix \ref{proof-circ} that for $k=1,\dots,n$ 
\begin{equation*}
  \Delta_{p_k}^{(h_0)}(t)  \xrightarrow{t \to \infty}:= p_k{(h_1)}(t) =  [r \sin (\omega t+\alpha_k), r \cos(\omega t+\alpha_k)]^T
\end{equation*}
where
\begin{itemize}
\item $\omega=2 \sin(\frac{\pi}{n})$
\item $r= \sqrt{c_1^2+c_2^2}$, where 
\begin{eqnarray*}
c_1 &=&\frac{1}{n}\sum_{l=1}^n \left [ \begin{matrix}\cos(\frac{2 \pi}{n}l);  &  -\sin(\frac{2 \pi}{n}l) \end{matrix} \right ] p_l(0)\\
c_2 &=&\frac{1}{n}\sum_{l=1}^n\left [ \begin{matrix}\sin(\frac{2 \pi}{n}l);  &  \cos(\frac{2 \pi}{n}l) \end{matrix} \right ]p_l(0)
\end{eqnarray*}
\item $\displaystyle \alpha_{k+1}=\alpha_k +\frac{2 \pi}{n}$

\end{itemize}
Therefore, the agents asymptotically  move, equally spaced, on a circular orbit with a radius determined by the number of agents,$n$, and their initial positions $p_l(0);l=1,\dots,n$

 \end{enumerate}

\end{proof}

\subsection{Input induced group dynamics}\label{InpDyn}
We recall, see eq. (\ref{P-u}), that the position of the agents at time $t$ is given by 

\begin{equation*}
P(t) = P^{(h)}(t) + P^{(u)}(t)
\end{equation*}
where $P^{(h)}(t)$ is the homogeneous part of the solution, derived in section \ref{ZeroInp}, and $P^{(u)}(t)$ is the contribution of the guidance control, $U_c$. 
\begin{equation}\label{eq-Put}
 P^{(u)}(t) = \int_{0}^{t} e^{\hat{M}(t-\tau)} \hat{B} U_c \mathrm{d}\tau = \int_{0}^{t} \sum_{i=0}^{n-1} e^{\mu_i^{\pm}(t-\tau)} \hat{B} U_c \mathrm{d}\tau 
\end{equation}
where 
\begin{itemize}
\item $\hat{B}=B \otimes I_{2 \times 2}$
\item $B$ is the indicator vector to the ad-hoc leaders (the agents detecting the exogenous control, $U_c$)
\end{itemize}
Recalling the location of the eigenvalues of $\hat{M}$, see  Appendix \ref{eig-hat_M-1}, 
we can write
   \begin{equation}\label{eq-Pu}
 P^{(u)}(t) =   P^{(u_1)}(t) +  P^{(u_2)}(t)        + \Delta_P^{(u)}(t)
\end{equation}
where 
\begin{itemize}
\item $ P^{(u_1)}(t)$  is due to the elements with zero eigenvalue, existing for any $\theta$
\item $ P^{(u_2)}(t)$  is due to the elements with $\mu_1^-, \mu_{n-1}^+$ on the imaginary axis, existing if $\theta = \theta_c$
\item $ \Delta_P^{(u)}(t)$ is due to the remaining eigenvalues which have negative real part
\end{itemize} 

\subsubsection{Agents' asymptotic velocity }\label{V_i}
   \begin{eqnarray*}
 V(t) &=& \frac{d}{dt} \left[ P^{(u_1)}(t) +  P^{(u_2)}(t) + \Delta_P^{(u)}(t) \right ]\\
   &:=&V^{(1)}(t) + V^{(2)}(t) + V^{(3)}(t)      
\end{eqnarray*}

   \begin{eqnarray*}
 V^{(3)}(t) &=& \frac{d}{dt}  \Delta_P^{(u)}(t)\\
 &=& \begin{cases}
 \sum_{i=1}^{n-1} e^{\mu_i^{\pm}t} \hat{B} U_c  \quad & \text{if  } |\theta| < \theta_c       \\
 \left (e^{\mu_1^{+}t}+ e^{\mu_{n-1}^{-}t}+\sum_{i=2}^{n-1}  e^{\mu_i^{\pm}t} \right) \hat{B} U_c  \quad & \text{if  }  |\theta| = \theta_c       
 \end{cases}
\end{eqnarray*}

Since all $\mu_i$ in $V^{(3)}(t)$ have negative real part, we have  $$V^{(3)}(t)  \xrightarrow{t \to \infty}   \mathbf{0_{2n}} $$
Thus, the asymptotic velocity of agent $i$ is $$V_i(t \rightarrow \infty) = V_i^{(1)}(t \rightarrow \infty) + V_i^{(2)}(t \rightarrow \infty)  $$

\begin{lem}
Let $U_c$ be the broadcast velocity control,  $B$ the ad-hoc leaders pointer vector, $n$ the number of agents and  $n_l$ the number of agents detecting the broadcast control (ad-hoc leaders). Then,
\begin{enumerate}
\item if $|\theta| < \theta_c$ all agents move with asymptotic velocity $\displaystyle V_i(t \rightarrow \infty) =\frac{n_l}{n} U_c$.
\item if  $|\theta| = \theta_c$ then the asymptotic velocities of the agents are 
$$V_i(t \rightarrow \infty) = \frac{n_l}{n} U_c + V_i^c(t)  $$
 where $V_i^c(t)$ is the circular component of the velocity of agent $i$ at time $t$, such that
\begin{equation*}
V_i^c(t):= V_i^{(u_2)}(t \to \infty)= \left [ \begin{matrix}
r^V\sin(\omega t + \alpha_i^V)\\
r^V \cos(\omega t  + \alpha_i^V)
\end{matrix} \right]
\end{equation*}
where 
\begin{itemize}
\item $r^V$ is determined by $B$ and $U_c$ and is common to all agents
\item the phase shift between consecutive agents is $\displaystyle \frac{2 \pi}{n}$ 
\end{itemize}

\end{enumerate}

\end{lem}

\begin{proof}\mbox{}\\*

\begin{enumerate}

\item \underline{Derivation of $V_i^{(u_1)}(t)$}

\begin{equation}\label{eq-p-u1}
 P^{(u_1)}(t) = \frac{1}{2}\left[ \int_0^{t} (\zeta_0^+ (\zeta_0^+)^* +  \zeta_0^-  (\zeta_0^-)^* )\mathrm{d}\nu \right ] \hat{B} U_c 
\end{equation}
 $\hat{B}=B \otimes I_{2 \times 2}$.

\begin{equation}\label{eq-p-u1-1}
  P^{(u_1)}(t) = \frac{1}{2} t (\zeta_0^+ (\zeta_0^+)^* +  \zeta_0^-  (\zeta_0^-)^* )\hat{B} U_c
\end{equation}
Thus 
\begin{equation}\label{eq-V-u1-1}
  V^{(u_1)}(t) = \frac{1}{2}  (\zeta_0^+ (\zeta_0^+)^* +  \zeta_0^-  (\zeta_0^-)^* )\hat{B} U_c
\end{equation}
Using again $Y^{(0)}=  \zeta^+_0(\zeta^+_0)^* $ we obtain
\begin{equation}\label{eq-p-u1-2}
  V^{(u_1)}(t) =  \Re(Y^{(0)}) \hat{B} U_c 
\end{equation}
where $\Re(Y^{(0)})$ is given by (\ref{ReY}).

$$\Rightarrow  V_i^{(u_1)}(t \to \infty) = \frac{n_l}{n} U_c; \quad i=1,\dots,n$$

\item \underline{Derivation of $V_i^c(t)$}

\begin{equation}\label{eq-p-u2}
 P^{(u_2)}(t) = \frac{1}{2}\left[ \int_0^{t} (e^{\mu^+_{n-1}\nu}\zeta_{n-1}^+ (\zeta_{n-1}^+)^* + e^{\mu^-_{1}\nu} \zeta_1^-  (\zeta_1^-)^* )\mathrm{d}\nu \right ] \hat{B} U_c 
\end{equation}
Thus
\begin{eqnarray}
V^{(u_2)}(t) &=& \frac{d}{dt}  P^{(u_2)}(t)\\
   &=& \frac{1}{2}\left[ e^{\mu^+_{n-1}t}\zeta_{n-1}^+ (\zeta_{n-1}^+)^* + e^{\mu^-_{1}t} \zeta_1^-  (\zeta_1^-)^*  \right ] \hat{B} U_c \label{V-circ}
\end{eqnarray}
We note the similarity of eq. (\ref{V-circ}) to (\ref{P-circ}) with $\hat{B} U_c$ replacing $P(0)$. 
Therefore, we can write (by similarity)
\begin{equation}\label{Vi-circ}
V_i^c(t):= V_i^{(u_2)}(t \to \infty)= \left [ \begin{matrix}
r^V\sin(\omega t + \alpha_i^V)\\
r^V \cos(\omega t  + \alpha_i^V)
\end{matrix} \right]
\end{equation}
where
\begin{itemize}
\item $\displaystyle \omega=2 \sin(\frac{\pi}{n})$
\item $r^V = \sqrt{(c_1^V)^2+(c_2^V)^2}$, such that
\begin{eqnarray*}
c_1^V &=&\frac{1}{n}\sum_{l=1}^n b_l \left [ \begin{matrix}\cos(\frac{2 \pi}{n}l);  &  -\sin(\frac{2 \pi}{n}l) \end{matrix} \right ]  U_c\\
c_2^V &=&\frac{1}{n}\sum_{l=1}^n  b_l \left [ \begin{matrix}\sin(\frac{2 \pi}{n}l);  &  \cos(\frac{2 \pi}{n}l) \end{matrix} \right ] U_c
\end{eqnarray*}
\item $\displaystyle \alpha_{i+1}^V=\alpha_i^V +\frac{2 \pi}{n}$
\end{itemize}

\end{enumerate}
\end{proof}

\subsubsection{Contribution of remaining eigenvalues: $\Delta_P^{(u)}(t \to \infty)$}\label{Derive-Dev}

The eigenvalues whose contribution has not been considered so far are \textbf{the eigenvalues with negative real part}, causing element $\Delta_P^{(u)}$ in (\ref{eq-Pu}),  due to the integration in $P^{(u)}(t)$, eq. (\ref{eq-Put}). These eigenvalues are
\begin{itemize}
\item  all eigenvalues other than $\mu_0^{\pm}$, if $\theta < \theta_c$
\item all eigenvalues other than $\mu_0^{\pm}, \mu_1^-, \mu_{n-1}^+$ , if $\theta = \theta_c$ 
\end{itemize}
Let $p_c(t)$ denote the moving centroid at time $t \to \infty$, such that
$$p_c(t)=p_c(0)+\frac{n_l}{n}U_c$$
where $p_c(0)$ is the centroid of the initial positions. Then 
\begin{equation}\label{p_i-d}
p^d_i(t)=p_c(t)+\Delta_{p_i}^{(u)}
\end{equation}
where $\Delta_{p_i}^{(u)}; \quad i \in [1,\dots,n]$, such that $\Delta_P^{(u)}=(\Delta_{p_1}^{(u)}, \dots, \Delta_{p_n}^{(u)})^T $, represents the asymptotic deviation  of agent $i$ from $p_c(t)$. 
\begin{itemize}
\item if $\theta < \theta_c$ then $p_i^d(t)$ is the  position of  agent $i$ at time $t \to \infty$ 
\item if $\theta = \theta_c$ then $p_i^d(t)$ is the center of the circular orbit for agent $i$  at time $t \to \infty$
\end{itemize}
   
\begin{lem}\label{L-Delta-p_u}
Let $p^d_i(t)$ be defined as in eq. (\ref{p_i-d}). Then $p^d_i(t); i= 1, \dots, n$ are aligned on a line rotated by angle $-\theta$ from the direction of $U_c$.  The dispersion of the agents along this line is a function of the set of agents detecting the exogenous control,$U_c$, represented by the vector $B$.

\end{lem}

\begin{proof}\

Consider the case \underline{\textbf{ $\theta < \theta_c$ }}. Then we have
\begin{eqnarray*}
  \Delta_p^{(u)}(t) & = &\frac{1}{2 } \left[ \sum_{i=1}^{n-1} (\int_0^{t}( \zeta_i^+ e^{\mu_i^+ \nu} (\zeta_i^+)^* +  \zeta_i^- e^{\mu_i^-\nu} (\zeta_i^-)^*) \mathrm{d}\nu )  \right ]\hat{B} U_c\\
      & = & \frac{1}{2} \sum_{i=1}^{n-1} [\frac{1}{\mu^+_i}(1-e^{\mu_i^+ t})\zeta_i^+(\zeta_i^+)^* + \frac{1}{\mu^-_i}(1-e^{\mu_i^- t})\zeta_i^-(\zeta_i^-)^*   ]\hat{B} U_c\\
      & \triangleq & \Delta_P^{(u_1)}+\Delta_P^{(u_2)}(t)\label{eq-Delta-p_u}
\end{eqnarray*}
where
\begin{eqnarray}
\Delta_P^{(u_1)} & = &  \frac{1}{2} \sum_{i=1}^{n-1}[\frac{1}{\mu^+_i}\zeta_i^+(\zeta_i^+)^* + \frac{1}{\mu^-_i}\zeta_i^-(\zeta_i^-)^*   ]\hat{B} U_c\label{eq-Delta-p_u1}\\
\Delta_P^{(u_2)}(t)  & =& - \frac{1}{2} \sum_{i=1}^{n-1}[\frac{1}{\mu^+_i}e^{\mu_i^+ t}\zeta_i^+(\zeta_i^+)^* + \frac{1}{\mu^-_i}e^{\mu_i^- t}\zeta_i^-(\zeta_i^-)^*   ]\hat{B} U_c\label{eq-Delta-p_u2}
\end{eqnarray}
Since $\Re(\mu_i)<0; i=1, \dots, n-1$  we have $\Delta_P^{(u_2)}(t) \xrightarrow{t \to \infty} \mathbf{0}_{2 n}$ and therefore $$\Delta_P^{(u)}(t) \xrightarrow{t \to \infty} \Delta_P^{(u_1)}$$.

Recalling that (see  Appendix \ref{eig-hat_M-1})
\begin{eqnarray*}
\mu_i^{\pm} & = &\lambda_i e^{\pm j \theta)}\\
\zeta_i^\pm  & = & v_i \otimes (1,\pm j)^T
\end{eqnarray*}
where $ \lambda_i, v_i; \quad i=0, \dots, n-1$ are the eigenvalues and eigenvectors of $M$, respectively, and
 using the property of product of two Kronecker products, eq. (\ref{Kron3}), we have 
\begin{equation*}
\frac{1}{\mu^+_i}e^{\mu_i^+ t}\zeta_i^+(\zeta_i^+)^*  =  \frac{e^{-j\theta}}{\lambda_i}v_i v_i^* \otimes \left[ 
\begin{matrix}
1& -j\\
j & 1
\end{matrix} \right ]
\end{equation*}
Similarly
\begin{equation*}
\frac{1}{\mu^-_i}e^{\mu_i^- t}\zeta_i^-(\zeta_i^-)^*  =  \frac{e^{j\theta}}{\lambda_i}v_i v_i^* \otimes \left[ 
\begin{matrix}
1& j\\
-j & 1
\end{matrix} \right ]
\end{equation*}
Let 
\begin{equation}\label{def-A}
A = e^{-j \theta}
             \begin{bmatrix}
           1 & -j \\
           j & 1
          \end{bmatrix}
\end{equation}

Then,
\begin{equation*}
\frac{1}{\mu^+_i}e^{\mu_i^+ t}\zeta_i^+(\zeta_i^+)^* + \frac{1}{\mu^-_i}e^{\mu_i^- t}\zeta_i^-(\zeta_i^-)^*  =  \frac{2}{\lambda_i}v_i v_i^* \otimes \Re(A)
\end{equation*}
where
 \begin{align}
\Re(A) =
          \begin{bmatrix}
           \cos(\theta) & -\sin(\theta) \\
           \sin(\theta) & \cos(\theta)
          \end{bmatrix}
          = R(-\theta)
  \end{align}
where $R(\theta)$ is the rotation matrix defined in (\ref{eq-R}).
Thus, we can rewrite eq. 
(\ref{eq-Delta-p_u1}) as
\begin{eqnarray}
  \Delta_P^{(u_1)} &=&  \sum_{i=1}^{n-1} (\frac{1}{\lambda_i}  v_i v_i^* \otimes R(-\theta)) (B \otimes I) U_c\\
            &=&  ([\sum_{i=1}^{n-1} \frac{1}{\lambda_i}  v_i v_i^*]B) \otimes R(-\theta) U_c\label{Dev-vec}
\end{eqnarray}
 Given $\lambda_i, v_i; i=0, \dots, n-1$ as defined in (\ref{eig-cyclic}), (\ref{eig-vec}) respectively and $M =circ(-1,1, 0, \dots, 0)$ we have 
  \begin{equation*}
  v_i^* M^T= \overline{\lambda_i}  v_i^*
  \end{equation*}
Thus, we can rewrite (\ref{Dev-vec}) as
\begin{eqnarray}
 \Delta_P^{(u_1)} & = & ([\sum_{i=1}^{n-1} \frac{1}{\lambda_i \overline{\lambda_i}}  v_i v_i^*] M^T B) \otimes R(-\theta) U_c \\
 & := & (\gamma  M^T B) \otimes R(-\theta) U_c\label{Dev-vec-gamma}
 \end{eqnarray}
where 
\begin{equation}\label{eq-gamma}
\gamma = [\sum_{i=1}^{n-1} \frac{1}{\lambda_i \overline{\lambda_i}}  v_i v_i^*]
\end{equation}
is an $n \times n$ matrix common to all agents.
Let 
\begin{equation}\label{eq-sigma}
\sigma = \gamma M^T B
\end{equation}

 Then we have 
\begin{equation}\label{Dev-p_i}
 \Delta_{p_i}^{(u_1)} = \sigma_i R(-\theta) U_c
\end{equation}
where $\sigma_i$ is a scalar. Thus, the agents asymptotically align in a linear formation rotated by an angle $-\theta$ from the direction of $U_c$ . The dispersion of the agents along this line is a function of the set of agents detecting the exogenous control, represented by $B$.

Consider now the case \underline{\textbf{ $\theta = \theta_c$ }}.\\
In this case there are two additional eigenvalues of $\hat{M}$ with zero real part, $\mu_1^-$ and $\mu_{n-1}^+$, and following the same procedure as above, we have 
\begin{eqnarray*}
\Delta_P^{(u_1)} & = &   \frac{1}{2} \left [  \sum_{i=1}^{n-1}[\frac{1}{\mu^+_i}\zeta_i^+(\zeta_i^+)^* -+\frac{1}{\mu^-_i}\zeta_i^-(\zeta_i^-)^*   ] - \frac{1}{\mu^-_1}\zeta_1^-(\zeta_1^-)^* - \frac{1}{\mu_{n-1}^+}\zeta_{n-1}^+(\zeta_{n-1}^+)^* \right ] \hat{B} U_c\\
  & = & (\tilde{\gamma} M^T B) \otimes R(-\theta) U_c
\end{eqnarray*}
where $$\tilde{\gamma} = \gamma - \frac{1}{\lambda_1 \overline{\lambda_1}}  v_1 v_1^*$$ and $\gamma$ is given by (\ref{eq-gamma}).
\begin{equation}\label{Dev-p_i-2}
 \Delta_{p_i}^{(u_1)} = \tilde{\sigma}_i R(-\theta) U_c
\end{equation}
where
\begin{equation}\label{eq-tilde-sigma}
 \tilde{\sigma} = \tilde{\gamma} M^T B
\end{equation}
\end{proof}

\begin{remark}\
If $B=\mathbf{1}_n$, i.e. the signal $U_c$ is detected by all the agents, then $\Delta_{p_i}^{(u)}=0; \forall i$  and therefore $p_i^d(t)=p_c(t); \forall i$. This is due to  $M^T B = \mathbf{0}_n$ in the expression for $\sigma$ and $\tilde{\sigma}$, eqs. (\ref{eq-sigma}) and (\ref{eq-tilde-sigma}), respectively. 
\end{remark}

\section{ Demonstration by Simulation }\label{Sim} 
In this section we illustrate by simulation the impact of the deviation angle $\theta$, of the broadcast control, $U_c(t)$, and of the set of leaders (agents detecting the broadcast control) on the emergent behavior of a group of  agents in linear, deviated, cyclic pursuit. We recall that the set of agents detecting the broadcast control is represented by the vector $B$, such that $B(i)=1$ if agent $i$ detected the broadcast signal and $B(i)=0$ otherwise.

We present results for 
\begin{enumerate}
\item Single time interval
\begin{enumerate}
\item Autonomous deviated linear cyclic pursuit, without external control
\item The broadcast control and the set of leaders are constant in the time interval, 
i.e. $U_c(t)=U_c, B(t)=B$ for $t \in [0,T_{max}]$
\end{enumerate} 
\item Multiple time intervals, where the time line $[0,T_{max}]$ is divided in piecewise constant time intervals $[t_k, t_{k+1}]; \quad k=0, ....m-1$, where $m$ is the number of intervals and
\begin{itemize}
\item $U_c(t) = U_c(t_k)$ if $t \in [t_k,t_{k+1})$
\item $B(t) = B(t_k)$ if $t \in [t_k,t_{k+1})$
\item The start state of a new interval is the end state of the previous interval 
\end{itemize}
\end{enumerate} 

In all shown simulation results we assumed $n=5$, where $n$ is the number of agents. 

\subsection{Autonomous deviated cyclic pursuit}
In this section we illustrate the emergent behavior of the autonomous deviated cyclic pursuit, i.e. with $U_c=(0,0)^T$, as a function of the deviation angle $\theta$, corresponding to the analytically obtained results (see section \ref{ZeroInp}).
We recall that $\displaystyle \theta_c=\frac{2 \pi}{n}$.  Therefore, for $n=5$ we have $\theta_c=36deg$.

We show, starting from two initial topologies, shown in Figures \ref{Fig-Init-n5-1}, (see Example1)  and  \ref{Fig-Init-n5-2},(see Example2), that 
\begin{itemize}
\item if $\theta < \theta_c$  then the agents gather to $p_c$, the  centroid of the agents' initial positions, see Figures \ref{Fig-Traj-n5-1-th20-u0} and \ref{Fig-Traj-n5-Ex2-th20-u0}. In these figures
\begin{itemize}
\item $p_c$ is the displayed initial centroid
\item $p_i(T_{max}):=(x_i(T_{max}, y_i(T_{max}))^T$
\item $p_i(T_{max}) ==p_c \quad \forall i $, as expected. 
\end{itemize}
\item if $\theta = \theta_c$ then the agents will rotate around $p_c$ on a common orbit with a radius (displayed as $r$) depending on the initial positions of the agents, see Figures \ref{Fig-Traj-n5-1-th36-u0} and \ref{Fig-Traj-n5-Ex2-th36-u0}.

In these figures we show the trajectories of the agents,\\ $p_i(t); i=1, \dots, n; t \in [0, T_{max}]$ and the circular component of the trajectories (shown in black). We note that the center of the circular orbit is $p_c$, as expected.
\end{itemize}

\subsubsection{Example1}
\begin{figure}[H]
  \centering
    \includegraphics[scale=0.7]{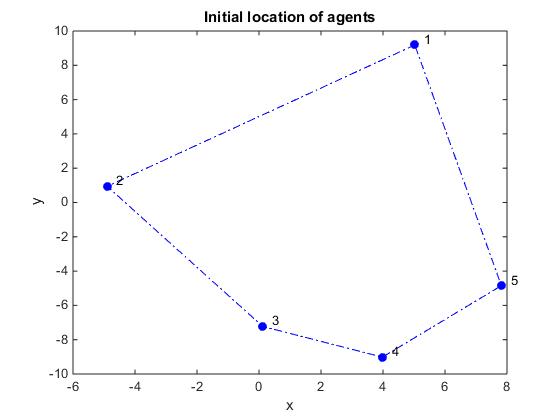}
    \caption{Example1 - Initial positions of the agents }\label{Fig-Init-n5-1}
\end{figure}

\begin{figure}[H]
  \centering
    \includegraphics[scale=0.7]{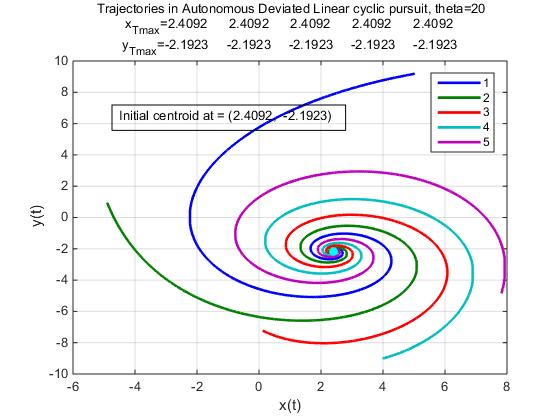}
    \caption{$\theta <\theta_c$ - Agents gathering to the centroid of the initial positions }\label{Fig-Traj-n5-1-th20-u0}
\end{figure}

\begin{figure}[H]
  \centering
    \includegraphics[scale=0.7]{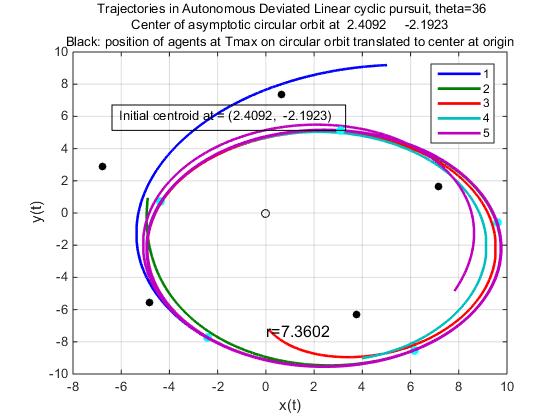}
    \caption{$\theta =\theta_c$ - Agents rotating around the centroid of the initial positions }\label{Fig-Traj-n5-1-th36-u0}
\end{figure}

\subsubsection{Example2}
\begin{figure}[H]
  \centering
    \includegraphics[scale=0.7]{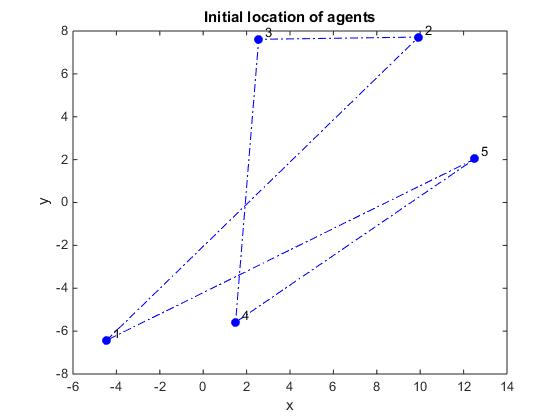}
    \caption{Example2 - Initial positions of the agents }\label{Fig-Init-n5-2}
\end{figure}

\begin{figure}[H]
  \centering
    \includegraphics[scale=0.7]{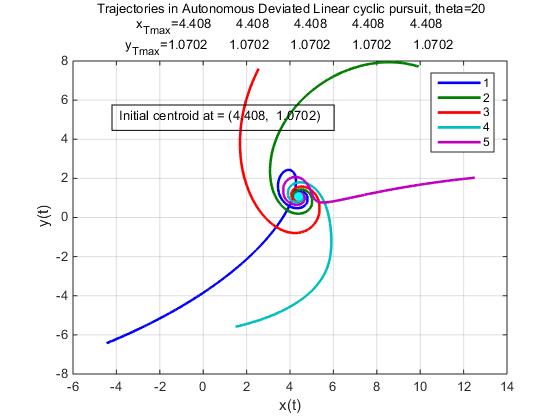}
    \caption{$\theta <\theta_c$ - Agents gathering to the centroid of the initial positions }\label{Fig-Traj-n5-Ex2-th20-u0}
\end{figure}

\begin{figure}[H]
  \centering
    \includegraphics[scale=0.7]{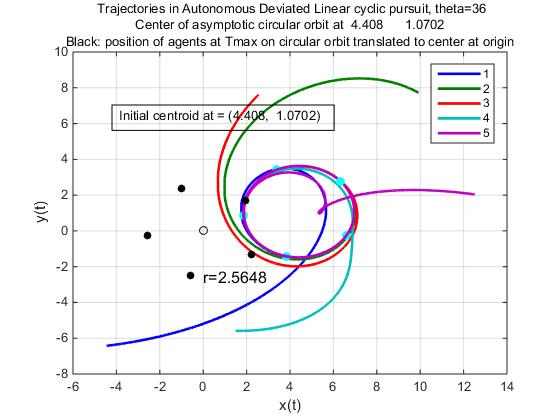}
    \caption{$\theta =\theta_c$ - Agents rotating around the centroid of the initial positions }\label{Fig-Traj-n5-Ex2-th36-u0}
\end{figure}

\subsection{The impact of the broadcast control and of the random set of leaders }
The remaining simulations presented in this section were run with the initial locations shown in Figure \ref{Fig-Init-n5-2}.

In all the figures included in this section 
\begin{itemize}
\item a solid line represents an ad-hoc leader (an agent that detects the broadcast control)
\item a dotted line represents a follower (an agent that does not detect the broadcast control)
\end{itemize}

\subsubsection{Single time interval}
We recall that the emergent behavior, analytically derived in section \ref{P_t_inf}, is an asymptotic behavior. Thus, in order to compare the numerical results, obtained by simulation to the analytically derived results we need the simulation time, $T_{max}$ to be long enough to approximate  $t \rightarrow \infty$. We found that $T_{max} = 60$ is such a time. Therefore all presented simulations were run with $T_{max} = 60$ (Points=60000, dT=0.001). However, for visualization purposes all trajectories and velocities are shown for the first part of the data, up to  $t=10$.
   
For the simulations presented in this section we used  $U_c=(2,3)^T$ and we considered the following cases
\begin{enumerate}
\item $\theta=20deg (< \theta_c)$
\begin{enumerate}
\item  $U_c$ is detected by all the agents, i.e. $B=\mathbf{1}_n$
\item $U_c$ is detected only by agents 2 and 5, i.e $B=(0,1,0,0,1)^T$
\end{enumerate}
\item $\theta=36deg (= \theta_c)$
\begin{enumerate}
\item  $U_c$ is detected by all the agents, i.e. $B=\mathbf{1}_n$
\item $U_c$ is detected only by agents 2 and 5, i.e $B=(0,1,0,0,1)^T$
\end{enumerate}
\end{enumerate}

\newpage
\underline{Case 1.1: $\theta = 20deg, \quad B=\mathbf{1}_n$ }
\begin{figure}[H]
  \centering
    \includegraphics[scale=0.75]{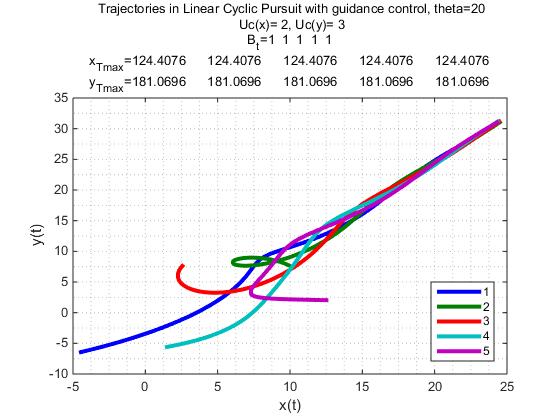}
    \caption{Case 1.1 - Emergent trajectories  }\label{Fig-Traj-Ball}
\end{figure}

\begin{figure}[H]
  \centering
    \includegraphics[scale=0.75]{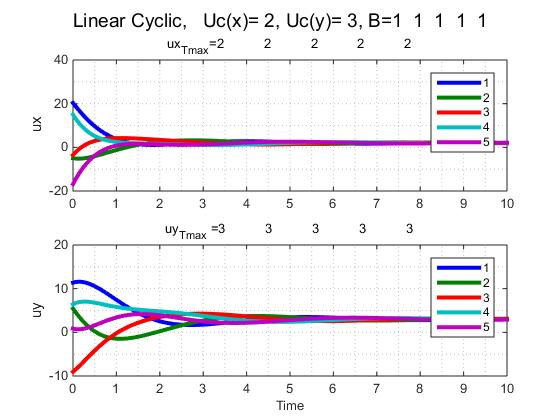}
    \caption{Case 1.1 - Emergent velocities }\label{Fig-Vel-Ball}
\end{figure}

In this case all the agents converge to a point (see $x_{T_{max}},y_{T_{max}}$ in Fig. \ref{Fig-Traj-Ball})  that moves with asymptotic velocity $U_c $,(see $ux_{T_{max}},uy_{T_{max}}$ displayed in Fig. \ref{Fig-Vel-Ball}), as expected. 

\newpage
\underline{Case 1.2: $\theta = 20deg, \quad B=(0,1,0,0,1)^T$ }

\begin{figure}[H]
  \centering
    \includegraphics[scale=0.75]{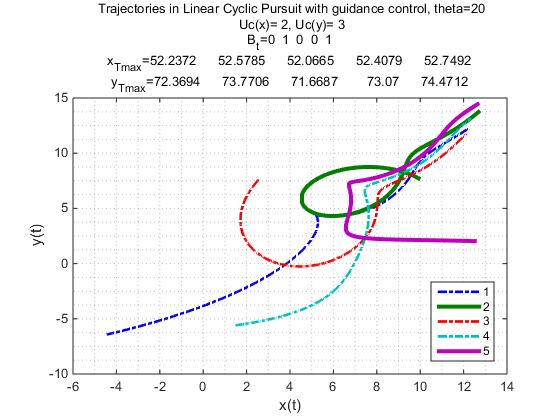}
    \caption{Case1.2 - Emergent trajectories }\label{Fig-Traj-B01001}
\end{figure}

\begin{figure}[H]
  \centering
    \includegraphics[scale=0.75]{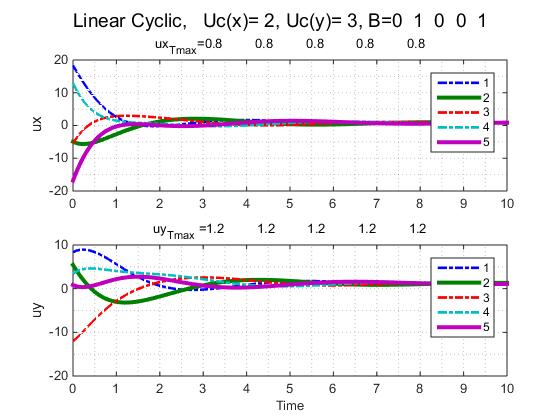}
    \caption{Case1.2 - Emergent velocities  }\label{Fig-Vel-B01001}
\end{figure}

In this case 
\begin{itemize}
\item the number of ad-hoc leaders, $n_l$ is two.
\item the velocity of each agent asymptotically converges to $\displaystyle \frac{n_l}{n}U_c$,\\ see $(ux_{T_{max}},uy_{T_{max}})$ in Fig.\ref{Fig-Vel-B01001}
\item the emergent trajectories of the agents are parallel lines in the direction of $U_c$, anchored at $p_c+\Delta_{p_i}^{(u_1)}$, (see Fig. \ref{Fig-Traj-B01001}) where $Delta_{p_i}^{(u_1)}$ is the asymptotic, time independent, deviation given by eq. (\ref{Dev-p_i})
\end{itemize} 

We complete the demonstration of the emergent behavior of the agents in case $\theta < \theta_c, B \neq \mathbf{1}_n$ by showing the deviations of the agents from a moving center for $\theta=20 \text{deg   vs   } 0 deg$, see Fig. \ref{Fig-Align-B01001-th20vs0}. We observe that in both cases the agents are aligned in a linear formation. While for $\theta=0$ the formation is aligned with $U_c$ for $\theta=20$the formation is aligned with $R(-\theta) U_c $, where $R$ denotes the rotation matrix, complying with eq.(\ref{Dev-p_i}).

\begin{figure}[H]
  \centering
    \includegraphics[scale=0.75]{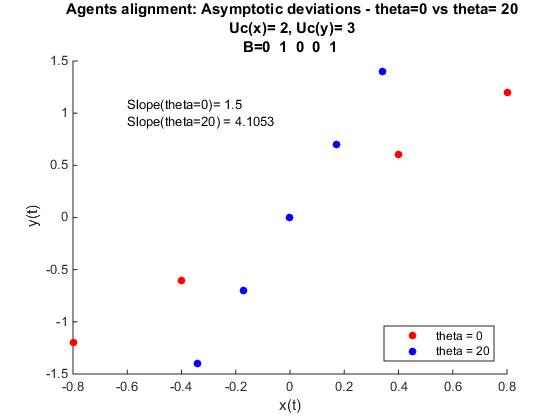}
    \caption{Agents' asymptotic linear formation, $\theta=20 \text{deg   vs   0deg}$ }\label{Fig-Align-B01001-th20vs0}
\end{figure}

In this case
\begin{equation*}
R(-20)=\left[
\begin{matrix}
0.9397 &  -0.3420\\
 0.3420 &   0.9397
\end{matrix} \right ]
\end{equation*}
Thus, $R(-\theta) U_c$ has a slope= 4.1053, as indicated.  

\newpage
\underline{Case 2.1: $\theta = \theta_c = 36 deg, \quad B=\mathbf{1}_n$ }
\begin{figure}[H]
  \centering
    \includegraphics[scale=0.75]{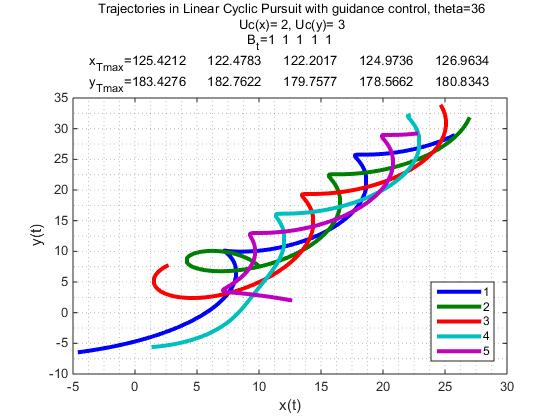}
    \caption{Case 2.1 - Emergent trajectories  }\label{Fig-Traj-Ball-th_c}
\end{figure}

\begin{figure}[H]
  \centering
    \includegraphics[scale=0.75]{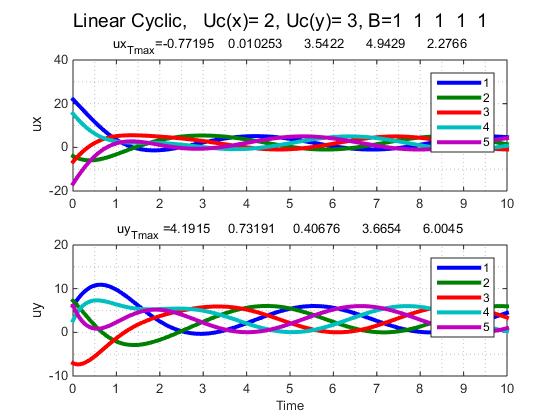}
    \caption{Case 2.1 - Emergent velocities }\label{Fig-Vel-Ball-th_c}
\end{figure}

\newpage
\underline{Case 2.2: $\theta = \theta_c = 36 deg, \quad B=(0,1,0,0,1)^T$ }

\begin{figure}[H]
  \centering
    \includegraphics[scale=0.75]{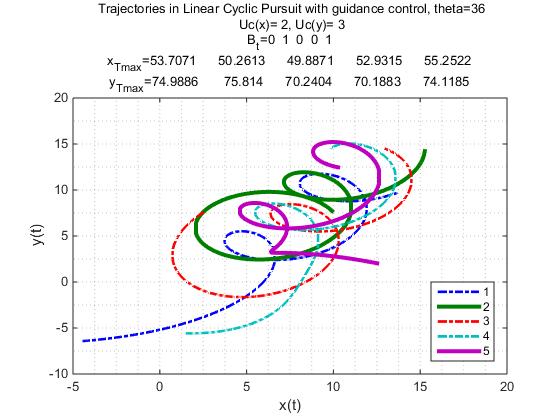}
    \caption{Case 2.2 - Emergent trajectories } \label{Fig-Traj-B01001-th_c}
\end{figure}

\begin{figure}[H]
  \centering
    \includegraphics[scale=0.75]{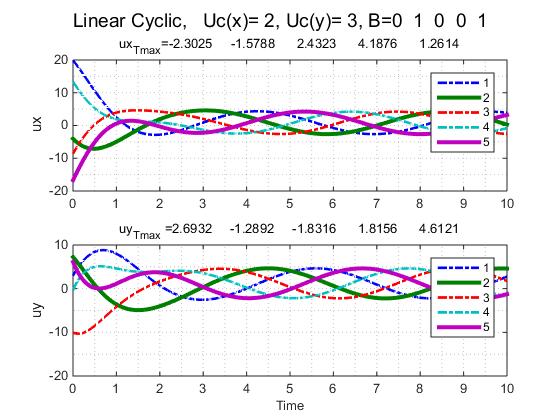}
    \caption{Case 2.2 - Emergent velocities}\label{Fig-Vel-B01001-th_c}
\end{figure}

\subsubsection{Multiple piecewise constant intervals}\label{Sim-multi}
A new interval starts upon a change of  broadcast velocity control, $U_c$, or a change in the set of agents detecting it, represented by the vector $B(t)$. The initial state of each interval is the end state of the previous interval.

In this section we show the emergent behavior of the same group of five agents, starting at the same initial positions, Fig. \ref{Fig-Init-n5-2}, over multiple intervals, with piecewise constant $U_c(t)$ and $B(t)$.  

We show two cases of  piecewise constant systems, $U_c(t)=(u_x(t),u_y(t))^T$ and $B(t)$, and for each system we show the emergent behavior, trajectories and velocities, for  $\theta < \theta_c$ and $\theta = \theta_c$. 

\newpage
\begin{enumerate}[label=\textbf{CaseM.\arabic*}]
\item $U(t)$ as shown in Fig. \ref{Fig-MultiUc-20-Ex1}, $B(t)$ as follows
\begin{itemize}
\item for $t \in [0,45)$ the set of agents detecting the broadcast control is $\{2,5\}$, i.e. $B(t)=(0,1,0,0,1)^T$
\item for $t \in [45,60]$ the set of  agents detecting the broadcast control is "all", i.e. $B(t) =\mathbf{1}_5$
\end{itemize}
\item $U(t)$ as shown in Fig. \ref{Fig-MultiUc-30-Ex2}, $B(t)$ as follows
\begin{itemize}
\item for $t \in [0,45)$ the set of agents detecting the broadcast control is $\{2,5\}$, i.e. $B(t)=(0,1,0,0,1)^T$
\item for $t \in [45,52.5)$ the set of  agents detecting the broadcast control is "all", i.e. $B(t) =\mathbf{1}_5$
\item for $t \in [52.5,60]$ the set of agents detecting the broadcast control is again $\{2,5\}$, i.e. $B(t)=(0,1,0,0,1)^T$
\end{itemize}

\end{enumerate}

We observe that the emergent behavior over multiple time intervals is a concatenation of the single intervals, as expected.

\begin{figure}[H]
  \centering
    \includegraphics[scale=0.65]{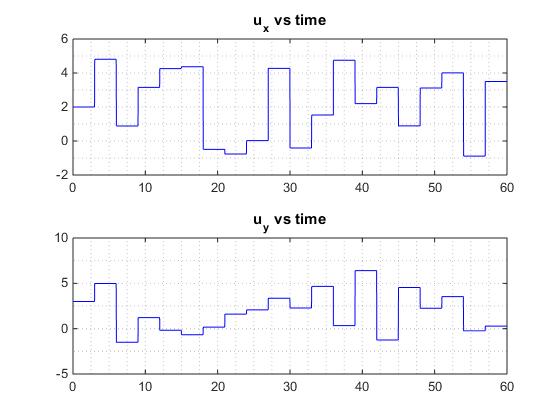}
    \caption{Piecewise constant broadcast control, $U_c(t)$ - \textbf{CaseM.1}}\label{Fig-MultiUc-20-Ex1}
\end{figure}

\begin{figure}[H]
  \centering
    \includegraphics[scale=0.65]{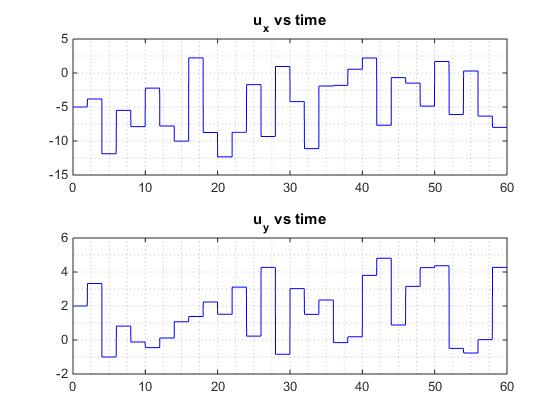}
    \caption{Piecewise constant broadcast control, $U_c(t)$ - \textbf{CaseM.2}}\label{Fig-MultiUc-30-Ex2}
\end{figure}

\subsubsection{CaseM.1.1 - CaseM.1 with  $\theta < \theta_c$ }

\begin{figure}[H]
  \centering
    \includegraphics[scale=0.85]{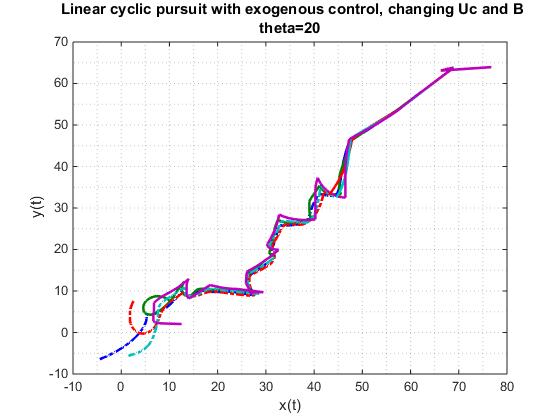}
    \caption{CaseM.1.1 - Emergent trajectories }\label{Fig-Traj-multi20u2B}
\end{figure}

\begin{figure}[H]
  \centering
    \includegraphics[scale=0.85]{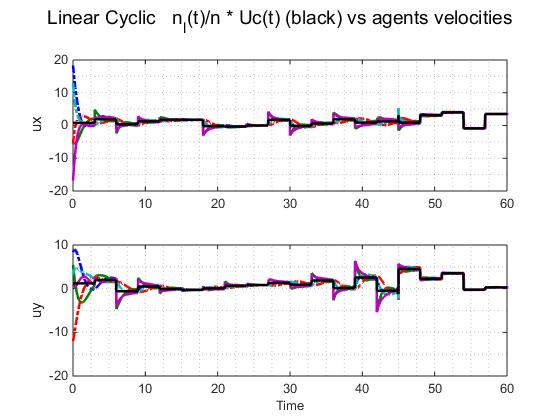}
    \caption{CaseM.1.1 - Emergent velocities  }\label{Fig-Vel-multi20u2B}
\end{figure}

\subsubsection{CaseM.1.2 - CaseM.1 with $\theta = \theta_c$ }

\begin{figure}[H]
  \centering
    \includegraphics[scale=0.85]{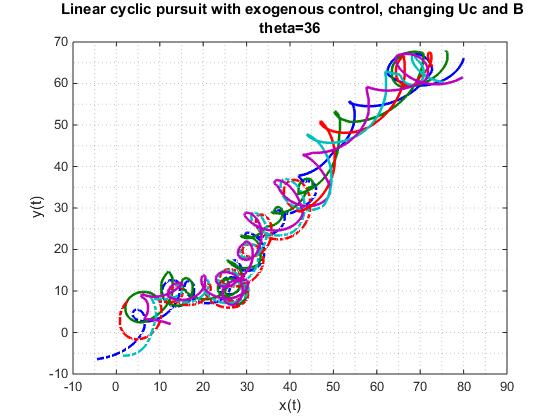}
    \caption{CaseM.1.2 - Emergent trajectories}\label{Fig-Traj-multi20u2B-thc}
\end{figure}

\begin{figure}[H]
  \centering
    \includegraphics[scale=0.85]{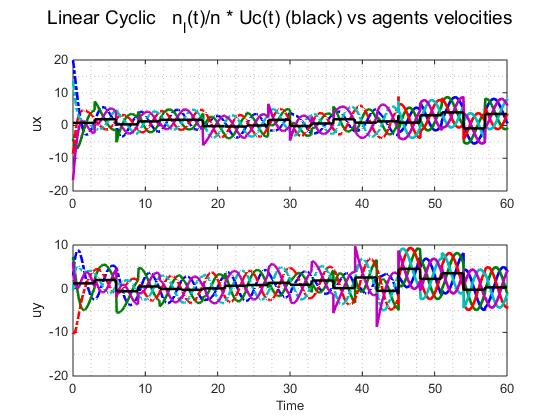}
    \caption{CaseM.1.2 - Emergent velocities }\label{Fig-Vel-multi20u2B-thc}
\end{figure}

\subsubsection{CaseM.2.1 - CaseM.2 with  $\theta < \theta_c$ }

\begin{figure}[H]
  \centering
    \includegraphics[scale=0.85]{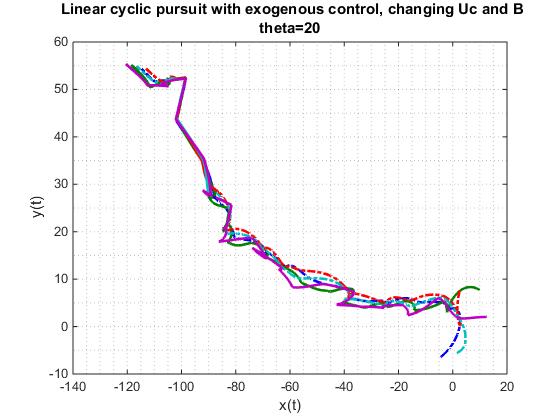}
    \caption{CaseM.2.1 - Emergent trajectories }\label{Fig-Traj-multi30u3B-Ex2}
\end{figure}

\begin{figure}[H]
  \centering
    \includegraphics[scale=0.85]{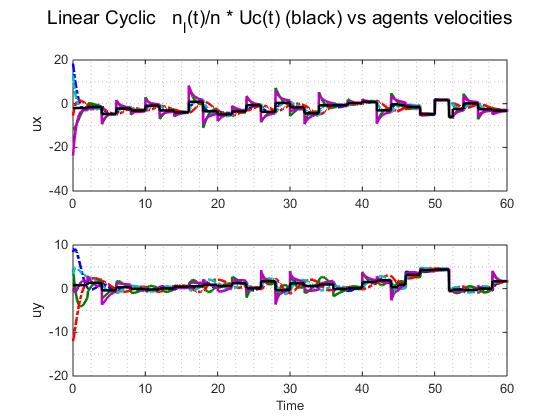}
    \caption{CaseM.2.1 - Emergent velocities  }\label{Fig-Vel-multi30u3B-Ex2}
\end{figure}

\subsubsection{CaseM.2.2 - CaseM.2 with $\theta = \theta_c$ }

\begin{figure}[H]
  \centering
    \includegraphics[scale=0.85]{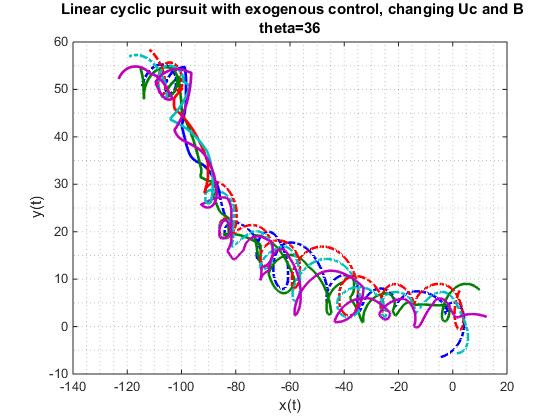}
    \caption{CaseM.2.2 - Emergent trajectories}\label{Fig-Traj-multi30u3B-thc-Ex2}
\end{figure}

\begin{figure}[H]
  \centering
    \includegraphics[scale=0.85]{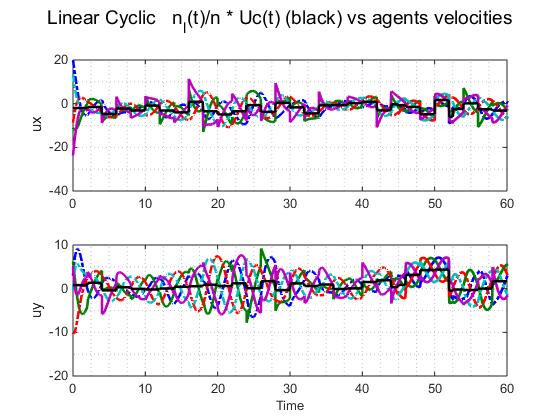}
    \caption{Example 2.2 - Emergent velocities }\label{Fig-Vel-multi30u3B-thc-Ex2}
\end{figure}

We note that what seems like an irregular circular movement in case $\theta=\theta_c$ and $B \neq \mathbf{1}_n$, see Fig.  \ref{Fig-Traj-multi20u2B-thc},  \ref{Fig-Traj-multi30u3B-thc-Ex2}, is actually due to changes in the deviation of the centers upon change in $U_c(t)$. (We recall that the deviation of the centers is given by eq. \ref{Dev-p_i}: $ \Delta_{p_i}^{(u_1)} = \sigma_i R(-\theta) U_c $). To demonstrate this claim we show the emergent behavior with the same $U_c(t)$ (see Fig. \ref{Fig-MultiUc-30-Ex2})  in case $B =\mathbf{1}_n$, when the centers overlap, i.e. there are no deviations.
\begin{figure}[H]
  \centering
    \includegraphics[scale=0.85]{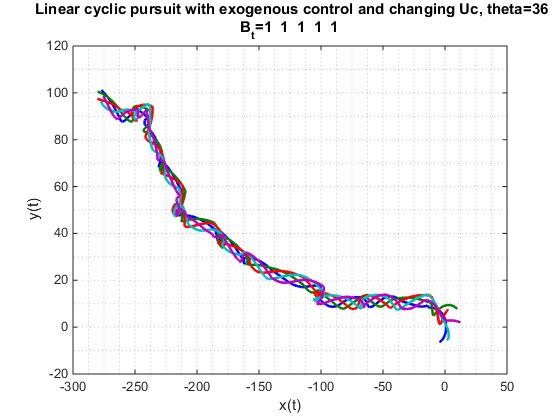}
    \caption{ Emergent trajectories when $\theta=\theta_c$ and $B=\mathbf{1}_n$}\label{Fig-Traj-multi30u1Ball-thc-Ex2}
\end{figure}

\begin{figure}[H]
  \centering
    \includegraphics[scale=0.85]{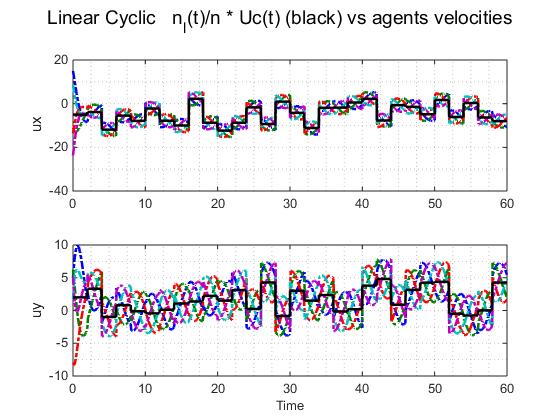}
    \caption{Emergent velocities when $\theta=\theta_c$ and $B=\mathbf{1}_n$ }\label{Fig-Vel-multi30u1Ball-thc-Ex2}
\end{figure}

\section{Summary and discussion}
In this work we apply well known LTI systems evolution theory to a \emph{novel paradigm} where agents in deviated linear cyclic pursuit are exposed to a piecewise constant broadcast control, detected by a random subset  
of agents. The emergent behavior of the agents is rigorously derived for a single time interval and illustrated by simulations for multiple time intervals, a concatenation of single time intervals.
The emergent behavior of the swarm is 
 shown to be a function of the deviation angle, $\theta$, the critical deviation angle $\displaystyle \theta_c=\frac{\pi}{n}$ and of the (random) subset of agents detecting the broadcast control (ad-hoc leaders). 
We show that the emergent pattern of movement depends only on the deviation angle $\theta$,  such that 
\begin{itemize}
\item if $\theta < \theta_c$ the movement is linear and the asymptotic velocity of the agents is $\displaystyle \frac{n_l}{n}U_c$, where $U_c$ is the broadcast control and $\displaystyle \frac{n_l}{n}$ is the ratio of agents detecting it. 
\begin{itemize}
\item If $\displaystyle \frac{n_l}{n}=1$ then the agents gather and move as a single point. 
\item If $\displaystyle \frac{n_l}{n}<1$ then the agents will asymptotically move as a time independent linear formation rotated by $R(-\theta)$ from the direction of $U_c$.
\end{itemize}
\item if $\theta =\theta_c$ the movement is circular around a center moving with velocity $\displaystyle \frac{n_l}{n}U_c$. The radius of the circular orbit is a function of the number of agents and the initial positions and is common to all agents.  
\begin{itemize}
\item If $\displaystyle \frac{n_l}{n}=1$  then the centers coincide. In this case the agents move, equally spaced, on a common orbit.  
\item If $\displaystyle \frac{n_l}{n}< 1$ then the centers move on parallel lines, such that at each point in time, $t$, they are placed on a line rotated by $R(-\theta)$ from the direction of $U_c$. The dispersion of the centers on this line is time-independent.
\end{itemize}

\end{itemize}

We plan in the future to investigate the emergent behavior of non-linear  agents, sensing bearing-only (bugs), in deviated cyclic pursuit, under the influence of broadcast control detected by a random subset of agents.

\newpage

\begin{appendices}
\section{General results on matrices}\label{matrices}
Following \cite{HJbook}, let $\mathbf{M}_n$ denote the class of all $n \times n$ matrices. Two matrices $A, B \in M_n$ are similar, denoted by $A \sim B$, if there exists an invertible (non-singular) matrix $S \in M_n$ s.t. $A = SBS^{-1}$. Similar matrices are just different basis representation of a single linear transformation.Similar matrices have the same characteristic polynomial, c.f. Theorem 1.3.3 in \cite{HJbook} and therefore the same eigenvalues

\subsection{Algebraic and geometric multiplicity of eigenvalues}\label{EigMultiplicity}

Let $\lambda$ be an eigenvalue of an arbitrary matrix $A \in M_n$ with an associated eigenvector $v \in \mathbb{C}^n$.\\
\textbf{Definitions:}
\begin{itemize}
  \item The spectrum of $A \in M_n$  is the set of all the eigenvalues of $A$, denoted by $\sigma ( A )$.
  \item The spectral radius of $A$ is $\rho ( A ) = max {| \lambda | : \lambda \in \sigma ( A ) }$.
  \item  For a given $\lambda \in \sigma(A)$ , the set of all vectors $v \in \mathbb{C}^n$ satisfying $A v = \lambda v$ is called the eigenspace of $A$ associated with the eigenvalue $\lambda$ . Every nonzero element of this eigenspace is an eigenvector of $A$ associated with $\lambda$
  \item The \emph{algebraic multiplicity} of $\lambda$ is its multiplicity as a root of the characteristic polynomial $det(\lambda I - A)$
  \item The \emph{geometric multiplicity} of  $\lambda$ is the dimension of the eigenspace associated with $\lambda$, i.e. the number of linearly independent eigenvectors associated with that eigenvalue.
  \item  We say that $\lambda$ is simple if its algebraic multiplicity is 1; it is semisimple if its algebraic and geometric multiplicities are equal.
  \item The algebraic multiplicity of an eigenvalue is larger than or equal to its geometric multiplicity.
  \item  We say that $A$ is defective if the geometric multiplicity of some eigenvalue is less than its algebraic multiplicity
\end{itemize}

\subsection{Diagonizable matrices}\label{DiagMat}

\textbf{Definition}: If $A \in M_n$ is similar to a diagonal matrix, then $A$ is said to be diagonizable

\begin{theorem}\label{th-diagonizable}
See Theorem 1.3.7 in \cite{HJbook}.

The matrix $A \in M_n$ is diagonizable iff there are $n$ linearly independent vectors, $v^{(1)},v^{(2)}, ....,v^{(n)}$,  each of which is an eigenvector of $A$. If $v^{(1)},v^{(2)}, ....,v^{(n)}$ are linearly independent eigenvectors of $A$ and $S=[v^{(1)},v^{(2)}, ....,v^{(n)} ]$ then $S^{-1} A S$ is a diagonal matrix $\Lambda$ and the diagonal entries of $\Lambda$ are the eigenvalues of $A$
\end{theorem}

\begin{lemma}\label{distinct-k}
Let $\lambda_1,...,\lambda_k; \quad k \geq 2$ be distinct eigenvalues of $A \in M_n$ and suppose $v^{(i)}$ is an eigenvector associated with $\lambda_i; \quad i=1,....,n$. Then the vectors $[v^{(1)},v^{(2)}, ....,v^{(k)} ]$ are linearly independent.
\end{lemma}
Proof of Lemma 1.3.8 in \cite{HJbook}

\begin{theorem}\label{distinct-all}
If $A \in M_n$ has $n$ distinct eigenvalues, then $A$ is diagonizable
\end{theorem}
\begin{proof}
Since all eigenvalues are distinct Lemma \ref{distinct-k} ensures that the associated eigenvectors are linearly independent and thus, according to Theorem \ref{th-diagonizable}, $A$ is diagonizable
\end{proof}
\textbf{Notes:}
\begin{enumerate}
  \item Having distinct eigenvalues is sufficient for diagonizability, but not necessary.
  \item A matrix is diagonizable iff it is non-defective, i.e. it has no eigenvalue with geometric multiplicity strictly less than its algebraic multiplicity
\end{enumerate}

\subsection{Left eigenvectors}\label{LeftEigvec}
\textbf{Definition}: A non-zero vector $y \in \mathcal{C}^n$ is a left eigenvector of $A \in M_n$ associated with eigenvalue $\lambda$ of $A$ if $y^*A = \lambda y^*$.
From \cite{HJbook}, Theorem 1.4.12, we have the following relationship between left and right eigenvectors and the multiplicities of the corresponding eigenvalue:
\begin{theorem}\label{left-right}
Let $\lambda \in \mathcal{C}$ be an eigenvalue of $A \in M_n$ associated with right eigenvector $x$ and left eigenvector $y^*$. Then the following hold:
\begin{enumerate}[label=(\alph*)]
\item If $\lambda$ has algebraic multiplicity 1, then $y^*x \neq 0$
\item If $\lambda$ has geometric multiplicity 1, then it has algebraic multiplicity 1 iff $y^*x \neq 0$
\end{enumerate}
\end{theorem}

\subsection{Square matrices decomposition}\label{App-decomp}

\begin{itemize}

 \item  If $A \in \mathbf{M}_n$ with distinct eigenvectors (not necessarily distinct eigenvalues) then $A=V \Lambda V^{-1}$, where $\Lambda$ is a diagonal matrix formed from the eigenvalues of $A$, and the columns of $V$ are the corresponding eigenvectors of $A$.\\
     A matrix $A \in \mathbf{M}_n$ always has $n$ eigenvalues, which can be ordered (in more than one way) to form a diagonal matrix $\Lambda \in \mathbf{M}_n$ and a corresponding matrix of nonzero columns $V$ that satisfies the eigenvalue equation $AV=V\Lambda$. If the $n$ eigenvectors are distinct then $V$ is invertible, implying the decomposition $A=V \Lambda V^{-1}$.\\
     Comment: The condition of having $n$ distinct eigenvalues is sufficient but not necessary. The necessary and sufficient condition is for each eigenvalue to have geometric multiplicity equal to its algebraic multiplicity.
 \item If $A$ is real-symmetric its $n$  (possibly not distinct) eigenvalues are all real with geometric multiplicity which  equals the algebraic multiplicity.
 $V$ is always invertible and can be made to have normalized columns. Then the equation $VV^T=I$ holds, because each eigenvector is orthonormal to the other.
 Therefore the decomposition (which always exists if $A$ is real-symmetric) reads as: $A=V \Lambda V^T$. This is known as the \emph{the spectral theorem}, or
 \emph{symmetric eigenvalue decomposition} theorem.
 \item If $A \in \mathbf{M}_n$ is normal, i.e. $A A^T = A^T A$, then
  \begin{enumerate}
    \item There exists an orthonormal set of $n$ eigenvectors of $A$
    \item $A$ is unitarily diagonizable, i.e. $A = U \Lambda U^*$, where $U$ is a unitary matrix of eigenvectors and $\Lambda$ is a diagonal matrix of eigenvalues of $A$.
  \end{enumerate}
 \end{itemize}

\subsection{About non-symmetric real matrices}\label{non-sym}
\begin{itemize}
  \item the eigenvalues of non-symmetric real $n\times n$ matrix are real or come in complex conjugate pairs
  \item the eigenvectors are not orthonormal in general and may not even span an n-dimensional space
      \begin{itemize}
        \item Incomplete eigenvectors can occur only when there are degenerate eigenvalues, i.e. eigenvalues with algebraic multiplicity greater than 1, but do not always occur in such cases
        \item Incomplete eigenvectors never occur for the class of normal matrices
      \end{itemize}
  \item Diagonalization theorem: an $n\times n$ matrix $A$ is diagonizable iff $A$ has $n$ linearly independent eigenvectors
\end{itemize}

\subsection{Normal matrices}\label{Normal}
\begin{definition}\label{NormalMat}
A matrix $A \in M_n$ is called normal if $A^* A = A A^*$
\end{definition}
\begin{definition}\label{Hermitian}
A matrix $A \in M_n$ is called Hermitian if $A^* =  A$
\end{definition}
\begin{theorem}\label{SpectralNormal}
If $A \in M_n$ has eigenvalues $\lambda_1, \lambda_2,.....,\lambda_n$ the following statements are equivalent:
\begin{enumerate}[label=(\alph*)]
\item $A$ is normal
\item $A$ is unitarily diagonizable
\item $\sum_{i=1}^n \sum_{j=1}^n |a_{ij}|^2 = \sum_{j=1}^n |\lambda_j|^2$
\item There is an orthonormal set of $n$ eigenvectors of $A$
\end{enumerate}
\end{theorem}

\begin{remark}
 All normal matrices are diagonizable but not all diagonizable matrices are normal.
\end{remark}

\subsection{Unitary matrices and unitary similarity}\label{UniMat}
Unitary matrices, $U \in M_n$, are non-singular matrices such that  $U^{-1} = U^*$, i.e $U^*U =U U^* =I$. A real matrix $U \in M_n(\mathbb{R})$ is real orthogonal if $U^T U=I$.
The following are equivalent:
\begin{enumerate}[label=(\alph*)]
\item $U$ is unitary
\item $U$ is non-singular and $U^{-1} = U^*$
\item $U U^* =I$
\item $U^*$ is unitary
\item The columns of $U$ are orthonormal
\item The rows of $U$ are orthonormal
\item For all $x \in \mathcal{C}^n, \quad \|x\|_2=\|Ux\|_2$
\end{enumerate}
\begin{defn}:
\begin{itemize}
  \item $A$ is unitarily similar to $B$ if there is a unitary matrix $U$ s.t. $A=U B U^*$
  \item $A$ is unitarily diagonizable if it is unitarily similar to a diagonal matrix
\end{itemize}
\end{defn}

\section{ Kronecker product  properties}\label{Kron-def}
This section follows mostly reference \cite{Kron}.
\begin{defn}
The Kronecker product of the matrix $A \in \mathbf{M}^{p,q}$ with the
matrix $B \in \mathbf{M}^{r,s}$ is defined as
 \begin{align}
 A  \otimes B  =
             \begin{bmatrix}
           A_{11}B & A_{12}B & \hdots & A_{1q}B \\
           A_{21}B & A_{22}B & \hdots & A_{2q}B \\
	   &&\vdots \\
           A_{p1}S & A_{p2}B & \hdots & A_{pq}B \\
          \end{bmatrix}
  \end{align}

\end{defn}
\subsection{Basic Properties}\label{Kron-Prop}
The properties of the Kronecker product in this subsection are only those with bearing on this work. A complete list of properties can be found in \cite{Kron}.
\begin{itemize}
  \item Taking the transpose before carrying out the Kronecker product yields the same result as doing so afterwards, i.e.
      \begin{equation}\label{Kron1}
        (  A  \otimes B)^T =  A^T  \otimes B^T
      \end{equation}\label{Kron2}
  \item The product of two Kronecker products yields another Kronecker product:
  \begin{equation}\label{Kron3}
    (A \otimes B) (C \otimes D) = AC \otimes BD; \quad \forall A \in \mathbf{M}^{p,q}, B \in \mathbf{M}^{r,s}, C \in \mathbf{M}^{q,k}, D \in \mathbf{M}^{s,l}
  \end{equation}
  \item Taking the complex conjugate before carrying out the Kronecker product yields the same result as doing so afterwards, i.e.
      \begin{equation}
        ( A \otimes B)^* =  A^* \otimes B^*; \quad \forall A \in \mathbf{M}^{p,q}(\mathbb{C}), B \in \mathbf{M}^{r,s}(\mathbb{C})
      \end{equation}\label{Kron4}
  \item Denote by $\sigma(A)$  the spectrum of a square matrix $A$, i.e. the set of all eigenvalues of $A$. Then, Theorem \ref{eig-Kron} holds.
  \begin{thm} (Theorem 2.3 in \cite{Kron})\label{eig-Kron}
Let $A \in \mathbf{M}^m$ and $B \in \mathbf{M}^n$. Furthermore, let $\lambda \in \sigma(A)$ with corresponding eigenvector $x$ and let $ \mu \in \sigma(B)$ with corresponding eigenvector $y$. Then, $\lambda \mu$ is an eigenvalue of $A  \otimes B$ with corresponding eigenvector $x \otimes y$
\end{thm}
\end{itemize}

\section{Properties of circulant matrices}\label{circulant}
A circulant matrix is an $ n \times n$ matrix having the form
\begin{equation}\label{eq-circ}
   C  =
          \begin{bmatrix}
           c_0 & c_1 & c_2  & \hdots & c_{n-1} \\
             c_{n-1} & c_0 & c_1 & \hdots & c_{n-2} \\
	   &&\vdots \\
            c_1 & c_2  && \hdots  & c_0 \\
          \end{bmatrix}
  \end{equation}

which can also be characterized as an $ n \times n$  matrix $C$ with entry $(k,j)$ given by
\begin{equation*}
  C_{k,l}=c_{(l-k) mod (n)}
\end{equation*}

Every $ n \times n$  circulant matrix $C$ has eigenvectors  (cf. \cite{Gray}, \cite{RamirezPhD})
\begin{equation}\label{eig-vec}
  v_k=\frac{1}{\sqrt{n}} \left ( 1, e^{-2 \pi jk/n}, e^{-4 \pi jk/n},\hdots, e^{-2 \pi jk(n-1)/n}  \right)^T; \quad k \in \left \{0, 1, \hdots, n-1 \right \}
\end{equation}
where $ j=\sqrt{-1}$, with corresponding eigenvalues
\begin{equation}\label{eig}
  \lambda_k = \sum_{l=0}^{n-1} c_l e^{-2 \pi j lk/n}
\end{equation}

From the definition of eigenvalues and eigenvectors we have
\begin{equation*}
  C v_k = \lambda_k v_k; \quad k= 0,1, \hdots, n-1
\end{equation*}
which can be written as a single matrix equation
\begin{equation*}
  C U= U \Lambda
\end{equation*}
where $\Lambda = diag(\lambda_k); \quad k=0, \hdots, n-1$ and
\begin{eqnarray*}
  U &=& \left [v_0, v_1, \hdots, v_{n-1} \right ] \\
   &=& \frac{1}{\sqrt{n}}\left [e^{-2\pi j mk/n}; \quad m,k = 0, \hdots, n-1 \right ]
\end{eqnarray*}
$U$ is a unitary matrix, i.e. $ UU^* = U^*U=I$ (cf. \cite{Gray}, proof by direct computation) and
\begin{equation}\label{C-diagonalization}
  C=U \Lambda U^*
\end{equation}
Note that $F_n = \sqrt{n} U^*$ is the known Fourier matrix.

\subsection{Cyclic linear pursuit}\label{cyclic}
The matrix $M$, representing cyclic pursuit,  is a special case of circulant matrix (see eq. (\ref{eq-M})
$M = circ(-1,1,0,0, \dots, 0)$.

Thus, using (\ref{eig}), the eigenvalues of $M$ are
\begin{equation}\label{eig-cyclic}
  \lambda_k = -1+  e^{-2 \pi j k/n}; \quad k=0, ..., n-1
\end{equation}
and the eigenvectors are given by eq. (\ref{eig-vec}).

\begin{lem}\label{normal-M}
$M = circ(-1,1,0,0, \dots, 0)$ is a normal matrix, i.e. satisfies $M M^T =  M^T M$
\end{lem}
\begin{proof}
By direct computation $M M^T = M^T M=circ([2, -1, 0, .... ,0, -1])$
\end{proof}

  \section{ Properties of the block circulant matrix  $\hat{M}$}\label{eig-hat_M}
  Consider now the block circulant matrix $\hat{M}$, defined as in eq. (\ref{M_hat}), representing linear cyclic pursuit with a common deviation angle $\theta$. 
 \begin{equation*}
\hat{M}= M \otimes R(\theta) 
\end{equation*}  
  Thus, $\hat{M}  = circ \left [ -R(\theta), R(\theta), 0_{2 \times 2}, ..., 0_{2 \times 2} \right ]$, where $R(\theta)$ is the rotation matrix, defined by eq. (\ref{eq-R})

   \begin{lem}\label{normal-hat_M}
 The matrix $\hat{M}=M \otimes R(\theta)$, where $M = circ(-1,1,0,0, \dots, 0)$ and $R$ is a rotation matrix, is a normal matrix
  \end{lem}

  \begin{proof}
  Using the Kroenecker product properties and Lemma \ref{normal-M} we prove that $\hat{M} \hat{M}^T = \hat{M}^T \hat{M}$.

  \begin{eqnarray*}
   \hat{M} \hat{M}^T &=& (M \otimes R) (M \otimes R)^T = (M M^T)\otimes (R R^T) = (M M^T)\otimes I \\
     \hat{M}^T \hat{M} &=& (M \otimes R)^T (M \otimes R) = (M^T M)\otimes I = (M M^T)\otimes I =  \hat{M} \hat{M}^T 
  \end{eqnarray*}

\end{proof}

\subsection{Eigenvalues and eigenvectors of $\hat{M}$}\label{eig-hat_M-1}
Denote by $\mu_i^{\pm}$ the eigenvalues of $\hat{M}$ with corresponding eigenvectors $ \zeta_i^{\pm} $.
Recalling that  $\hat{M}=(M \otimes R(\theta))$ and applying theorem \ref{eig-Kron}, the eigenvalues of $\hat{M}$ are products of the eigenvalues of $M$ and of $R(\theta)$, while the eigenvectors of $\hat{M}$ are Kronecker products of the eigenvectors of  $M$ and of $R(\theta)$

\begin{lem}
The rotation matrix $R(\theta)$ has 2 eigenvalues $e^{\pm j \theta}$ with corresponding eigenvectors $\varsigma^+=(1, j)^T,\quad \varsigma^- = (1,-j)^T$, where $j$ denotes the imaginary unit.
\end{lem}
Proof by direct computation.

Let $\lambda_i; i=0,...,n-1$ be the eigenvalues of $M$ with corresponding eigenvectors $v_i; i=0,...,n-1$.
According to (\ref{eig-cyclic}), we have for $k=0,...,n-1$

\begin{eqnarray*}
\lambda_k &=&- 1+  e^{-2 \pi j k/n}\\
     &=& -1+\cos(\frac{2 \pi k}{n})-j\sin(\frac{2 \pi k}{n})\\
     &=& -2 \sin^2(\frac{ \pi  k}{n})- 2j \sin(\frac{ \pi  k}{n}) \cos(\frac{ \pi  k}{n})\\
     & = & -2\sin(\frac{ \pi  k}{n}) \left (\sin(\frac{ \pi  k}{n}) + j \cos(\frac{ \pi k}{n}) \right )\\
     & = & -2\sin(\frac{ \pi  k}{n}) \left (\cos(\frac{\pi}{2}-\frac{ \pi  k}{n})+j\sin(\frac{\pi}{2}-\frac{ \pi  k}{n})\right )\\
     &=&  2\sin(\frac{ \pi  k}{n})\left (\cos(\frac{\pi}{2}+\frac{ \pi  k}{n})-j\sin(\frac{\pi}{2}+\frac{ \pi  k}{n})\right )\\
     & = & 2\sin(\frac{ \pi  k}{n}) e^{-j ( \frac{\pi}{2}+\frac{ \pi  k}{n})}
 \end{eqnarray*}

The following Lemma is well known in the linear algebra literature
\begin{lem}\label{conj-eig}
If $\lambda$ is a complex eigenvalue of a real matrix $A$, with corresponding, complex, eigenvector $v$, then $\overline{\lambda}$ is also an eigenvalue of $A$, with eigenvector $\overline{v}$
\end{lem}
\begin{proof}\
By the definition of $\lambda$ and $v$, we have
\begin{equation*}
Av = \lambda v,\quad  v \neq 0.
\end{equation*}
Taking complex conjugates of this equation, we obtain:
\begin{equation*}
\overline{A} \overline{v} = A \overline{v} = \overline{\lambda} \overline{v}
\end{equation*}
where $ \overline{A} =A$ since $A$ is real. Therefore,  
$\overline{\lambda}$ is also an eigenvalue of $A$, with eigenvector 
$\overline{v}$
\end{proof}

For the particular case of $A=M$, we have $\lambda_0=0$ and $\lambda_{n-1}=\overline{\lambda_1}$

The eigenvectors of $M$ are given by eq. (\ref{eig-vec}), i.e.
\begin{equation*}
  v_k= \frac{1}{\sqrt{n}}\left ( 1, e^{-2 \pi jk/n}, e^{-4 \pi jk/n},\hdots, e^{-2 \pi jk(n-1)/n}  \right)^T; \quad k \in \left \{0, 1, \hdots, n-1 \right \}
\end{equation*}
satisfying $v_0 =\mathbf{1}_n$ and $v_{n-1} = \overline{v}_1$.

Therefore, we have
\begin{corollary}\label{eig-M}
  The eigenvalues of $\hat{M}$ are $\mu_k^{\pm}; \quad k=0, \dots, n-1$, s.t.
\begin{equation}
\mu_k^{\pm} =\lambda_k e^{\pm j \theta}=2\sin(\frac{ \pi k}{n})e^{-j (\frac{\pi}{2}+\frac{ \pi  k}{n}\pm \theta)}
\end{equation}  
 with corresponding eigenvectors $ \zeta_k^+ = v_k \otimes \varsigma^+$ and $ \zeta_k^- = v_k \otimes \varsigma^-$.
 \end{corollary}
 \emph{Note}: The eigenvectors of $\hat{M}$ do not depend on the deviation angle $\theta$.\\
 
From Corollary \ref{eig-M}, it is immediate to see that 
\begin{enumerate}[label=\textbf{C.\arabic*},ref=C.\arabic*]
\item $\hat{M}$ has 2 zero eigenvalues, $\mu_0^\pm$, with corresponding eigenvectors $\zeta_0^\pm = [1, \pm j,1,\pm j, \dots, 1, \pm j]^T$,  
\item remaining eigenvalues depend on the value of $\theta$, while the eigenvectors are independent of $\theta$.
\begin{enumerate}
 \item $\mu_1^- = 2\sin(\frac{ \pi }{n}) e^{-j (\frac{\pi}{2}+\frac{ \pi  }{n}- \theta)}$ 
 \item $\mu_{n-1}^{+} = 2\sin(\frac{ \pi (n-1) }{n}) e^{-j (\frac{\pi}{2}+\frac{ \pi (n-1)}{n}+ \theta)} 
  = \overline{\mu_1^{-}}$
 \item $\zeta_{n-1}^{+} = v_{n-1} \otimes [1, j]^T =\overline{ v_1} \otimes \overline{[1, -j]}^T = \overline{\zeta_1^{-}}$, where we used the property of the Kronecker product for complex conjugates
 \end{enumerate}  
\end{enumerate}

\subsubsection{Critical deviation angle, $\theta_c$}\label{comp-theta_c}

 Let $\theta_c$ define the critical value of the deviation $\theta$, defined such that
$\Re(\mu^-_1)=\Re(\mu^+_{n-1})=0$.
\begin{eqnarray*}
\Re(\mu^+_{n-1}) &=& \cos(\frac{\pi}{2}+\frac{ \pi (n-1)  }{n} + \theta_c) =0 = \cos(\frac{3\pi}{2})\\
\Re(\mu^-_1) &=& \cos(\frac{\pi}{2}+\frac{ \pi   }{n} - \theta_c) =0= \cos(\frac{\pi}{2})
\end{eqnarray*}
Solving the above for $\theta_c$ we obtain $\displaystyle \theta_c=\frac{\pi}{n}$. Thus, for $\theta= \theta_c = \frac{\pi}{n}$ we have
\begin{itemize}
\item $\mu_1^{-} = 2 j \sin(\frac{\pi}{n})$
\item $\mu^+_{n-1} = -2 j \sin(\frac{\pi}{n})$ 
\end{itemize} 
\begin{corollary}\label{eig-M-2}\
For $n$ agents in linear cyclic pursuit with common deviation angle $\theta$, there exists a critical value  $\theta_c=\frac{\pi}{n}$, such that
\begin{itemize}
\item[(a)] if $|\theta| < \theta_c$, then $\hat{M}$ has two zero eigenvalues, $\mu_0^\pm$, and all non-zero eigenvalues of $\hat{M}$ lie in the open left-half complex plane
  \item[(b)] if  $|\theta| = \theta_c$, then $\hat{M}$ has two zero eigenvalues, $\mu_0^\pm$, and two non-zero eigenvalues which lie on the imaginary axis, while all remaining eigenvalues lie in the open left-half complex plane. The eigenvalues on the imaginary axis are 
  $\mu_{n-1}^+$ and  $\mu_1^-$       
  \item[(c)] if $|\theta| > \theta_c  $, then $\hat{M}$ has  two zero eigenvalues and at least two non-zero eigenvalues which lie in the open right-half complex plane, therefore this is an unstable system which shall not be discussed herein.
  
\end{itemize}
\end{corollary}

\section{Proof of Lemma \ref{L-Delta-p_h1}}\label{proof-circ}
We derive $p_k^{(h_1)}; k=1, \dots, n$ from
$$ P^{(h_1)}(t) = \frac{1}{2} [e^{\mu^+_{n-1} t} \zeta^+_{n-1} (\zeta^+_{n-1})^* + e^{\mu^-_1 t} \zeta^-_1 (\zeta^-_1)^*] P(0) $$. 
where $ P^{(h_1)}(t)=(p_1^{(h_1)}, p_2^{(h_1)},\dots, p_n^{(h_1)})^T$.  

We recall (see Corollary \ref{eig-M} in Appendix \ref{eig-hat_M-1}) that
\begin{eqnarray*}
\zeta_1^- &= &v_1 \otimes (1, -j)^T\\
\zeta_{n-1}^+ &= &v_{n-1} \otimes (1, j)^T
\end{eqnarray*}
Since $v_{n-1}= \overline{v_1}$ we have $\zeta_{n-1}^{+} =  \overline{\zeta_1^{-}}$    
\begin{eqnarray*}
 \mu_1^{-} &=& 2 j \sin(\frac{\pi}{n})\\
 \mu^+_{n-1} & = & -2 j \sin(\frac{\pi}{n}) 
\end{eqnarray*} 
Let $Y=\zeta_1^- (\zeta_1^-)^*$ and $\omega = 2 \sin(\frac{\pi}{n})$.
Then  

\begin{eqnarray}
 P^{(h_1)}(t) &=& \frac{1}{2}[e^{\omega j t}  Y + e^{-\omega j t} \overline{Y}]P(0)\\
 &=&\frac{1}{2} [(Y+ \overline{Y} )\cos (\omega t ) -j (Y- \overline{Y} )\sin (\omega t)]  P(0)\\
  & = & [\Re (Y) \cos (\omega t) +  \Im (Y)  \sin (\omega t)] P(0)\label{P_h1} 
\end{eqnarray}
Thus,
\begin{equation}\label{p_k_h1}
p_k^{(h_1)}(t)= \sum_{l=1}^n \{[\Re (Y)]_{kl} \cos (\omega t) +  [\Im (Y)]_{kl}  \sin (\omega t)\}p_l(0)
\end{equation}

where $  \Re (Y)$ and $\Im (Y)$ are the real matrices,   derived in Appendix \ref{t-dev}, and $[\Re (Y)]_{kl}$, $[\Im (Y)]_{kl}$ are $2 \times 2$ blocks of $\Re (Y), \Im (Y)$ matrices respectively.

\subsection{Derivation of $\Re (Y)$  and $\Im(Y)$}\label{t-dev}

\begin{equation*}
 Y=\zeta_1^- (\zeta_1^-)^* 
\end{equation*}
\begin{eqnarray*}
\zeta_1^- &=& v_1 \otimes [1, -j]^T\\
v_1 &=& \frac{1}{\sqrt{n}}[1, e^{-2 \pi j/n}, e^{-4 \pi j/n},  \dots , e^{-2 \pi j (n-1)/n}]^T
\end{eqnarray*}  
 Using the Kronecker properties (\ref{Kron1}) and (\ref{Kron3}) we obtain
  \begin{equation}\label{eq-Y}
Y= {v_1} v_1^* \otimes \left [ \begin{matrix}
 1 & j\\ 
-j & 1
\end{matrix}  \right ] 
\end{equation}
Denote $\displaystyle W={v_1} v_1^*$. Then 
\begin{equation*}
W_{kl} =\frac{1}{n} e^{-2 \pi j (k-l)/n} = \frac{1}{n} \left (\cos(2 \pi(k-l)/n) -j \sin(2 \pi(k-l)/n) \right ); \quad k,l = 1, \dots, n
\end{equation*}
Then
\begin{equation*}
Y=W \otimes \left [ \begin{matrix}
 1 & j\\ 
-j & 1
\end{matrix}  \right ] 
\end{equation*} 
and denote by $ [Y]_{kl}$ the $2 \times 2$ block of $Y$, such that 
\begin{equation*}
[Y]_{kl}=W_{kl} \left [ \begin{matrix}
 1 & j\\ 
-j & 1
\end{matrix}  \right ]; \quad k,l = 1, \dots, n
\end{equation*}
where $W_{kl}$ is a scalar.
Denote 
\begin{equation}\label{rho_kl}
\rho_{kl}= 2 \pi (k-l)/n
\end{equation}

Then 
\begin{equation*}
[Y]_{kl}=\frac{1}{n} \left (\cos(\rho_{kl}) -j \sin(\rho_{kl}) \right )  \left [ \begin{matrix}
 1 & j\\ 
-j & 1
\end{matrix}  \right ] = \frac{1}{n}\left [ \begin{matrix}
 \cos(\rho_{kl}) -j \sin(\rho_{kl})  & j\cos(\rho_{kl}) + \sin(\rho_{kl})\\ 
 -j \cos(\rho_{kl}) - \sin(\rho_{kl}) 
 & \cos(\rho_{kl}) -j \sin(\rho_{kl})
\end{matrix}  \right ]
\end{equation*}

Thus, 
\begin{equation}\label{Re_Y}
[\Re(Y)]_{kl}= \frac{1}{n}\left [ \begin{matrix}
 \cos(\rho_{kl}))  &  \sin(\rho_{kl})\\ 
 - \sin(\rho_{kl}) 
 & \cos(\rho_{kl}) 
\end{matrix}  \right ]
\end{equation}

\begin{equation}\label{Im_Y}
[\Im(Y)]_{kl}= \frac{1}{n}\left [ \begin{matrix}
  - \sin(\rho_{kl})  & \cos(\rho_{kl}) \\ 
 -\cos(\rho_{kl})  
 &  - \sin(\rho_{kl})
\end{matrix}  \right ]
\end{equation}

\subsection{Proof of circular movement for agent $k$}\label{circ_k}

Let
\begin{eqnarray}
c^{(a)}_{kl} &=& \frac{1}{n}\left [ \begin{matrix}\cos(\rho_{kl}))  &  \sin(\rho_{kl}) \end{matrix} \right ]\label{c_kl_a} \\
c^{(b)}_{kl} &=& \frac{1}{n}\left [ \begin{matrix}-\sin(\rho_{kl}))  &  \cos(\rho_{kl}) \end{matrix} \right ]\label{c_kl_b}
\end{eqnarray}
where $\rho_{kl}$ is defined by (\ref{rho_kl}).
Then eq. (\ref{p_k_h1}) can be re-written as

\begin{equation}\label{pk_1}
p_k^{(h_1)}(t) = \left [ \begin{matrix}
a_k \cos(\omega t)+b_k \sin(\omega t)\\
-a_k \sin(\omega t)+b_k \cos(\omega t)
 \end{matrix} \right ]
\end{equation}
where $$\omega=2 \sin(\frac{\pi}{n})$$
\begin{eqnarray}
a_k &=&\sum_{l=1}^n [c^{(a)}_{kl}p_l(0)]\label{a_k}\\
b_k &=& \sum_{l=1}^n [c^{(b)}_{kl}p_l(0)]\label{b_k}
\end{eqnarray}
where 
\begin{itemize}
\item $c^{(a)}_{kl}$ and $c^{(b)}_{kl} $   depend only on the number of agents $n$, see eqs (\ref{c_kl_a}), (\ref{c_kl_b}).
\item  $p_l(0); \quad l=1,\dots, n$ are the initial positions of the agents 
\end{itemize}
Therefore, $a_k, b_k$ are scalars \textbf{defined by the number of agents and the initial positions of the agents}. 
Equivalently, (\ref{pk_1}) can be written as
\begin{equation}\label{pk_1-1}
p_k^{(h_1)}(t) = \left [ \begin{matrix}
R_k\sin(\omega t + \alpha_k)\\
R_k \cos(\omega t  + \alpha_k)
 \end{matrix} \right ]
\end{equation}
where 
\begin{equation}\label{R_k}
R_k^2 = a_k^2 + b_k^2
\end{equation}
\begin{equation}\label{alpha_k}
\tan(\alpha_k) = \frac{a_k}{b_k}
\end{equation}
Next, we show that the radius is common to all agents and the agents are equally spaced on the orbit, at angular distances  $\displaystyle \frac{2 \pi}{n}$.

\subsection{Proof of movement on a common orbit }\label{circ_all}
We consider agents $k$ and $k+1$ and show that $R_{k+1}=R_k$ for $k=1, \dots,n$. Thus $R_k =R_1; \quad k=2, \dots, n$, i.e. $R_1$ is the common radius. 

\subsubsection{All orbits have a common radius}

\begin{equation}\label{R_k1}
R_{k+1}^2 = a_{k+1}^2 + b_{k+1}^2
\end{equation}

Given $\rho_{kl}= 2 \pi (k-l)/n$, see (\ref{rho_kl}), we have $$ \rho_{k+1,l}= \rho_{kl}+\frac{2 \pi}{n}$$. Therefore
\begin{eqnarray*}
c^{(a)}_{k+1,l} &=& \frac{1}{n}\left [ \begin{matrix}\cos(\rho_{k+1,l})  &  \sin(\rho_{k+1,l}) \end{matrix} \right ]\\
&=& \frac{1}{n}\left [ \begin{matrix}\cos(\rho_{kl}+\frac{2 \pi}{n}))  &  \sin(\rho_{kl}\frac{2 \pi}{n}) \end{matrix} \right ]\\
&=& \frac{1}{n}\left [ \begin{matrix} \cos(\rho_{kl})\cos(\frac{2 \pi}{n})
-\sin(\rho_{kl})\sin(\frac{2 \pi}{n}) & \sin(\rho_{kl})\cos(\frac{2 \pi}{n})+\cos(\rho_{kl})\sin(\frac{2 \pi}{n})\end{matrix} \right ]\\
&=& c^{(a)}_{kl}\cos(\frac{2 \pi}{n})+  c^{(b)}_{kl}\sin(\frac{2 \pi}{n})
\end{eqnarray*}

Similarly we obtain
$$c^{(b)}_{k+1,l} = c^{(b)}_{kl}\cos(\frac{2 \pi}{n})-  c^{(a)}_{kl}\sin(\frac{2 \pi}{n})$$

\begin{equation}\label{a_k1}
a_{k+1} = \sum_{l=1}^n c_{k+1,l}^{(a)} p_l(0)= a_k\cos (\frac{2 \pi}{n})+b_k \sin (\frac{2 \pi}{n})
\end{equation}

\begin{equation}\label{b_k1}
b_{k+1} =b_k\cos (\frac{2 \pi}{n})- a_k \sin (\frac{2 \pi}{n})
\end{equation}

Rewriting eq.(\ref{R_k1}) using (\ref{a_k1}) and (\ref{b_k1}) we obtain
\begin{eqnarray}
R_{k+1}^2 &=& \left( a_k\cos (\frac{2 \pi}{n})+b_k \sin (\frac{2 \pi}{n})\right)^2 +
\left( b_k\cos (\frac{2 \pi}{n})-a_k \sin (\frac{2 \pi}{n})\right)^2\\
 &=& a_k^2 \left( \cos^2(\frac{2 \pi}{n}) + \sin^2(\frac{2 \pi}{n}) \right)+b_k^2 \left( \sin^2(\frac{2 \pi}{n}) + \cos^2(\frac{2 \pi}{n}) \right)\\
 & & +2 a_k b_k \cos(\frac{2 \pi}{n}) - 2 a_k b_k \cos(\frac{2 \pi}{n})\\
  &=& a_k^2 + b_k^2 = R_k^2\label{R_k1-1}
\end{eqnarray}
Since (\ref{R_k1-1}) holds for $k= [1, \dots,n]  mod (n)$ we have
\begin{equation}\label{R_k1-2}
\Rightarrow \boxed{R_{k+1} = R_k = R_1} \quad \forall k \in [1, \dots,n]  mod (n)
\end{equation}
 
\subsubsection{Derivation of the common radius}
From (\ref{R_k1-2}) we see that $R_1$ is the common radius, $r$. Thus,$$r=\sqrt{a_1^2+b_1^2}$$.
\begin{eqnarray*}
a_1 &=&\sum_{l=1}^n [c^{(a)}_{1l}p_l(0)]\label{a_1}\\
b_1 &=& \sum_{l=1}^n [c^{(b)}_{1l}p_l(0)]\label{b_1}
\end{eqnarray*}
where
\begin{eqnarray}
c^{(a)}_{1l} &=& \frac{1}{n}\left [ \begin{matrix}\cos(\rho_{1l}))  &  \sin(\rho_{1l}) \end{matrix} \right ]\label{c_1l_a} \\
c^{(b)}_{1l} &=& \frac{1}{n}\left [ \begin{matrix}-\sin(\rho_{1l}))  &  \cos(\rho_{1l}) \end{matrix} \right ]\label{c_1l_b}
\end{eqnarray}
and $\displaystyle \rho_{1l}=\frac{2 \pi}{n}(1-l)$.
Using
\begin{eqnarray*}
\cos(\rho_{1l}) &=& \cos(\frac{2 \pi}{n})\cos(\frac{2 \pi}{n}l)+\sin(\frac{2 \pi}{n})\sin(\frac{2 \pi}{n}l)\\
\sin(\rho_{1l}) &=& \sin(\frac{2 \pi}{n})\cos(\frac{2 \pi}{n}l)+\cos(\frac{2 \pi}{n})\sin(\frac{2 \pi}{n}l)
\end{eqnarray*}
we obtain, after some algebra
\begin{eqnarray*}
a_1 &=&\cos(\frac{2 \pi}{n}) c_1 +  \sin(\frac{2 \pi}{n}) c_2\\
b_1  &=& -\sin(\frac{2 \pi}{n}) c_1 +  \cos(\frac{2 \pi}{n}) c_2  
\end{eqnarray*} 
where
\begin{eqnarray}
c_1 &=&\frac{1}{n}\sum_{l=1}^n\left [ \begin{matrix}\cos(\frac{2 \pi}{n}l)  &  -\sin(\frac{2 \pi}{n}l) \end{matrix} \right ]p_l(0)\label{c_1}\\
c_2 &=&\frac{1}{n}\sum_{l=1}^n\left [ \begin{matrix}\sin(\frac{2 \pi}{n}l)  &  \cos(\frac{2 \pi}{n}l) \end{matrix} \right ]p_l(0)\label{c_2}
\end{eqnarray}
Therefore we have
\begin{equation*}
r^2=a_1^2+b_1^2=c_1^2+c_2^2
\end{equation*}
\begin{equation}\label{eq-r}
\Rightarrow \boxed{r=\sqrt{c_1^2+c_2^2}}
\end{equation}
where $c_1, c_2$ are given by (\ref{c_1}),(\ref{c_2}), respectively.

\subsection{Equal spacing between agents on the orbit}
In this section we show that the angular distance between consecutive agents is $\frac{2 \pi}{n}$, i.e. we show $\alpha_{k+1} =\alpha_k+\frac{2 \pi}{n}$.

Corresponding to (\ref{alpha_k}) we have
\begin{equation}\label{alpha_k1}
\tan(\alpha_{k+1}) = \frac{a_{k+1}}{b_{k+1}}
\end{equation} 

Recalling (see eqs. (\ref{a_k1}) and (\ref{b_k1})
\begin{eqnarray*}
a_{k+1} &=& a_k\cos (\frac{2 \pi}{n})+b_k \sin (\frac{2 \pi}{n})\\
b_{k+1} &=& b_k\cos (\frac{2 \pi}{n})- a_k \sin (\frac{2 \pi}{n})
\end{eqnarray*}
\begin{equation}\label{alpha_k1-1}
\tan(\alpha_{k+1})= \frac{a_k\cos (\frac{2 \pi}{n})+b_k \sin (\frac{2 \pi}{n})}{b_k\cos (\frac{2 \pi}{n})- a_k \sin (\frac{2 \pi}{n})}
\end{equation}
From (\ref{alpha_k}) we have
\begin{eqnarray*}
\sin(\alpha_k) &=& \frac{a_k}{\sqrt{a_k^2+b_k^2}}\\
\cos(\alpha_k) &=& \frac{b_k}{\sqrt{a_k^2+b_k^2}}
\end{eqnarray*}
Thus we can rewrite (\ref{alpha_k1-1}) as
\begin{equation*}
\tan(\alpha_{k+1})= \frac{(\sqrt{a_k^2+b_k^2})\sin (\alpha_k+\frac{2 \pi}{n})}{(\sqrt{a_k^2+b_k^2})\cos (\alpha_k+\frac{2 \pi}{n})} = \tan (\alpha_k+\frac{2 \pi}{n})
\end{equation*}
\begin{equation}\label{alpha_k1-2}
\Rightarrow \boxed{\alpha_{k+1} =\alpha_k+\frac{2 \pi}{n}}
\end{equation}

\end{appendices}

\newpage
\bibliographystyle{plain}
\bibliography{BibGen}

\end{document}